\documentclass[a4paper,UKenglish,cleveref, autoref, thm-restate]{lipics-v2021}

\title{A Parallel Approach to Counting Exact Covers Based on Decomposability Property} 

\author{Liangda Fang}{Jinan University, China}{}{}{}

\author{Yaohui Luo}{Jinan University, China}{}{}{}

\author{Delong Li}{Jinan University, China}{}{}{}

\author{Xuanxiang Huang}{Nanyang Technological University, Singapore}{}{}{}

\author{Quanlong Guan}{Jinan University, China}{}{}{}

\authorrunning{L. Fang, Y. Luo, D. Li, X. Huang and Q. Guan} 

\Copyright{Liangda Fang, Yaohui Luo, Delong Li, Xuanxiang Huang, and Quanlong Guan} 



\nolinenumbers 

\EventEditors{}
\EventNoEds{2}
\EventLongTitle{29th International Conference on Theory and Applications of Satisfiability Testing (SAT 2026)}
\EventShortTitle{SAT 2026}
\EventAcronym{SAT}
\EventYear{2026}
\EventDate{20-23 July, 2026}
\EventLocation{Lisbon (ISCTE), Portugal}
\EventLogo{}
\SeriesVolume{42}
\ArticleNo{23}

\usepackage{xifthen}
\usepackage{ifthen}
\usepackage{graphicx}
\usepackage{amsmath}
\usepackage{amsthm}
\usepackage[linesnumbered,ruled,slide,boxed,vlined]{algorithm2e}
\usepackage{array}
\usepackage{float}
\usepackage{multirow}
\usepackage{siunitx}

\graphicspath{{figures/}}

\newcommand{\collection}{\mathcal{U}}
\newcommand{\subcollection}{\mathcal{S}}
\newcommand{\exactCoverSolutions}{\mathcal{E}}

\newcommand{\extendedGraphSet}{\mathcal{G}}

\newcommand{\universe}{U}
\newcommand{\varSet}{V}

\newcommand{\semantics}[1]{\langle #1 \rangle}
\newcommand{\set}[1]{\{ #1 \}}
\newcommand{\size}[1]{|#1|}
\newcommand{\true}{\top}
\newcommand{\false}{\bot}

\newcommand{\assign}{\leftarrow}
\newcommand{\orthJoin}{\sqcup}
\newcommand{\leftfield}[1][]{\ifthenelse{\isempty{#1}} {{\tt left}} {{\tt left}[{#1}]} }
\newcommand{\rightfield}[1][]{\ifthenelse{\isempty{#1}} {{\tt right}} {{\tt right}[{#1}]} }
\newcommand{\upfield}[1][]{\ifthenelse{\isempty{#1}} {{\tt up}} {{\tt up}[{#1}]} }
\newcommand{\downfield}[1][]{\ifthenelse{\isempty{#1}} {{\tt down}} {{\tt down}[{#1}]} }
\newcommand{\headfield}[1][]{\ifthenelse{\isempty{#1}} {{\tt header}} {{\tt header}[{#1}]} }
\newcommand{\sizefield}[1][]{\ifthenelse{\isempty{#1}} {{\tt size}} {{\tt size}[{#1}]} }
\newcommand{\vertexfield}[1][]{\ifthenelse{\isempty{#1}} {{\tt vertex}} {{\tt vertex}[{#1}]} }
\newcommand{\parentfield}[1][]{\ifthenelse{\isempty{#1}} {{\tt parent}} {{\tt parent}[{#1}]} }
\newcommand{\rowfield}[1][]{\ifthenelse{\isempty{#1}} {{\tt row}} {{\tt row}[{#1}]} }

\newcommand\ie{{\it i.e.}}

\begin{document}

\maketitle

\begin{abstract}
    \looseness=-1
    The exact cover problem is a classical NP-hard problem with broad applications in the area of AI.
 	Algorithm DXZ is a method to count exact covers representing by zero-suppressed binary decision diagrams (ZBDDs).
 	In this paper, we propose a zero-suppressed variant of decision decomposable negation normal form (in short, decision-ZDNNF), which is strictly more succinct than ZBDDs.
 	We then design a novel parallel algorithm, namely \textit{DXD}, which constructs a decision-ZDNNF representing the set of all exact covers.
 	Furthermore, we improve DXD by dynamically updating connected components.
 	The experimental results demonstrate that the improved DXD algorithm outperforms all of state-of-the-art methods.
\end{abstract}

\section{Introduction}
\label{sec:introduction}

\looseness=-1
Given a family $\collection$ of subsets over a universe $\universe$, an \textit{exact cover} is a subfamily $\subcollection \subseteq \collection$ such that every element of $\universe$ is contained in exactly one subset of $\subcollection$.
The exact cover problem is a well-known NP-hard problem \cite{Karp1972}. 
In the area of AI, a variety of problems, such as $n$-queens \cite{Knu2000}, Sudoku \cite{GunM2011} and graph coloring \cite{Koi2006}, can be formulated as the exact cover problem.
The exact cover problem can be represented in an incidence matrix, where each column denotes an element of $\universe$ and each row represents a subset in $\collection$. 
If the subset represented by row $r$ contains the element represented by column $c$, then the entry in row $r$ and column $c$ is $1$; otherwise it is $0$.

\looseness=-1
One of the effective methods to enumerate all exact covers is \textit{DLX}, proposed by Knuth~\cite{Knu2000}.
The main insight behind DLX is to perform depth-first backtracking on the incidence matrix via iteratively removing and recovering columns and rows of the matrix.
To implement efficient removing and recovering operations, DLX utilizes dancing link technique \cite{HitN1979} to represent the matrix.
However, DLX has the following two deficiencies: (1) it consumes substantial memory to store all exact covers when the number of solutions tends to be large; and (2) it incurs an unnecessary computational overhead as many duplicate subproblems are solved multiple times.
To overcome these two issues, Nishino et al.~\cite{NisYMN2017} proposed an improved algorithm, namely \textit{DXZ}.
It employs \textit{zero-suppressed binary decision diagrams (ZBDDs)} \cite{Min1993} as the compact representation of all exact cover solutions and maintains one memoization cache to store the partial solutions to previously computed subproblems.
In real-world applications, it often fails to construct a dancing link representing the family $\collection$ of subsets due to the prohibitively large cardinality of $\collection$. 
To address this space issue, Nishino et al.~\cite{NisYN2021} proposed an extension to ZBDDs, namely \textit{DanceDD}, to compactly represent $\collection$, and design an algorithm D3X based on removing and recovering operations on DanceDDs.

\looseness=-1
ZBDDs are the zero-suppressed variant of binary decision diagrams (BDDs) \cite{Akers1978,Bry1986}.
BDDs require each subformula to be a decision node that consists of a decision variable $v$, the positive formula $\alpha$ and the negative formula $\beta$.
The main difference between BDDs and ZBDDs lies in different reduction rules they adopt.
Different reduction rules result in different interpretations of BDDs and ZBDDs.
The most convenient interpretation for BDDs is a mapping from decision nodes into Boolean functions \cite{Bry1986} and for ZBDDs is a mapping from decision nodes into families of sets \cite{Min1994}.
Due to the simplicity of decision nodes, BDDs provide an efficient representation of Boolean functions that supports tractable Boolean operations, including bounded conjunction, bounded disjunction, satisfiability checking and model counting.

\looseness=-1
In a seminal paper, Darwiche~\cite{Dar2001} proposed a novel type of nodes for propositional language, namely \textit{decomposable node}.
A decomposable node consists of several sub-nodes $\beta_1, \cdots, \beta_n$ where the sets of variables of $\beta_i$'s are pairwise disjoint, representing the conjunction $\beta_1 \land \cdots \land \beta_n$. 
By combining decision nodes and decomposable nodes, Darwiche~\cite{Dar2002} later developed another language, namely \textit{decision decomposable negation normal form (decision-DNNF)}.
In theory, compared to BDDs, Beame et al.~\cite{BeaLRS2017} showed that decision-DNNFs are a more succinct representation (more precisely, there is a class of propositional formulas that have polynomial decision-DNNF size but require exponential BDD size).
This was empirically demonstrated in standard graph coloring and circuit benchmarks \cite{HuaD2007}.
The exhaustive DPLL architecture \cite{BirL1999} integrated with connected component-based decomposition \cite{JrP2000}, which is the basic architecture widely used in the majority of state-of-the-art propositional model counters, such as Cachet \cite{SangBB2004}, SharpSAT \cite{Thu2006}, D4 \cite{LagM2017}, SharpSAT-TD \cite{KorJ2021}, essentially transforms a CNF formula into decision-DNNF.

\looseness=-1
Inspired by the succinctness of decision-DNNF compared to BDDs, this paper proposes a novel method to count exact covers based on a zero-suppressed variant of decision-DNNF (in short, decision-ZDNNF).
The main contributions are the following.
(1) We define the syntax and semantics of decision-ZDNNFs.
From the syntactic perspective, decision-ZDNNFs extends ZBDDs by allowing decomposable nodes.
The semantics establishes the interpretation of decision-ZDNNFs in terms of families of sets, which serves as the theoretical foundation for counting exact covers.
(2) We design a novel algorithm, namely \textit{DXD}, which generates a decision-ZDNNF to represent all exact covers.
Algorithm DXD constructs decision nodes like Algorithm DXZ does, and builds decomposable nodes based on the connected components of the primal graph of the incidence matrix.
Thanks to the decomposability of decision-ZDNNFs, Algorithm DXD enables parallel counting. 
(3) We design a dynamic algorithm to update connected components.
Under vertices and edges deletions and insertions, this dynamic algorithm is more efficient than breadth-first search. 
(4) The experimental results show that DXD significantly outperforms the existing methods in terms of runtime.
In particular, decision-ZDNNFs exhibit succinctness advantage over ZBDDs in representing the set of exact cover solutions.

\begin{algorithm}[t]\small
	\caption{${\tt DXZ}(M)$}
	\label{alg:dxz}
	
	\KwIn{$M$: an incidence matrix}
	\KwOut{$\alpha$: a root ZBDD node representing all exact cover solutions for $M$}
	
	\lIf{$M$ is empty}
	{
		\Return{$\top$}
	}
	
	\lIf{${\tt Cache}({\tt Col}(M)) \neq nil$}
	{
		\Return{${\tt Cache}({\tt Col}(M))$}
	}
	
	$c \assign $ a column of $M$ with minimal number of interacted rows
	
	$\alpha \assign \false$
	
	${\tt Cover}(c)$
	
	\ForEach{interacted row $r$ of column $c$}
	{
		\lForEach{interacted column $j$ of row $r$}
		{			
			${\tt Cover}(j)$
		}
		
		$\beta \assign {\tt DXZ}(M)$
		
        \lIf{$\beta \neq \false$}
		{
			$\alpha \assign {\tt BuildDecisionNode}(r, \alpha, \beta)$
		}
		
		\lForEach{interacted column $j$ of row $r$}
		{
			${\tt Uncover}(j)$
		}\vspace*{-1mm}
	}
	
	${\tt Uncover}(c)$
	
	${\tt Cache}({\tt Col}(M)) \assign \alpha$

	\Return{$\alpha$}
	\vspace*{-1mm}
\end{algorithm}

\section{Algorithm DXZ}

\begin{algorithm}\small
	\caption{${\tt Cover}(c)$}
	\label{alg:cover}
	\KwIn{$c$: a column}
	
	$h \assign$ the header cell for column $c$

	$\leftfield[{\rightfield[h]}] \assign \leftfield[h]$ 
	
	$\rightfield[{\leftfield[h]}] \assign \rightfield[h]$
	
	
	
	$i \assign \downfield[h]$
	
	\While{$i \neq h$}{
		
		$j \assign \rightfield[i]$
		
		\While{$j \neq i$}
		{
			$\upfield[{\downfield[j]}] \assign \upfield[j]$
			
			$\downfield[{\upfield[j]}] \assign \downfield[j]$
			
			
			$\sizefield[{\headfield[j]}] \assign \sizefield[{\headfield[j]}] - 1$
			
			
			$j \assign \rightfield[j]$
			
		}
		
		$i \assign \downfield[i]$
		
	}
\end{algorithm}

\begin{algorithm}\small
	\caption{${\tt Uncover}(c)$}
	\label{alg:uncover}
	\KwIn{$c$: a column}
	
	$h \assign$ the header cell for column $c$
	
	$i \assign \upfield[h]$
	
	
	\While{$i \neq h$}
	{
		$j \assign \leftfield[i]$
		
		
		\While{$j \neq i$}
		{
			$\upfield[{\downfield[j]}] \assign j$
			
			$\downfield[{\upfield[j]}] \assign j$
			
			
			$\sizefield[{\headfield[h]}] \assign \sizefield[{\headfield[h]}] + 1$
			
			
			$j \assign \leftfield[j]$
			
		}
		
		$i \assign \upfield[i]$
		
	}
	
	$\leftfield[{\rightfield[h]}] \assign h$
	
	$\rightfield[{\leftfield[h]}] \assign h$
	
\end{algorithm}

\looseness=-1
Algorithm DXZ, illustrated in Algorithm \ref{alg:dxz}, is a recursive procedure that takes an incidence matrix $M$ as input and outputs all exact cover solutions
$\exactCoverSolutions$ for $M$.
It maintains one memoization cache so as to avoid recomputation.
This cache takes the set of remaining columns in the reduced matrix $M$ as key and outputs the address of a ZBDD node $\alpha$ that represents all exact cover solutions of $M$.
If matrix $M$ is empty, then the algorithm outputs $\true$, representing a power set of empty universe (line 1).
Then, it first determines if the current matrix has already been encountered or not (line 2).
If so, it directly returns the ZBDD node stored in the cache.

The algorithm selects a column $c$ from matrix $M$ with the minimal number of interacted rows (line 3) and initializes $\alpha$ as $\false$ (line 4).
We say row $r$ is an \textit{interacted} row of column $c$ (or column $c$ is an \textit{interacted} column of row $r$), iff $M_{r, c} = 1$.
In the following, it stipulates that the union of any exact cover solution contains $c$.
Then, it covers column $c$ from $M$, that is, removing column $c$ together with all covering rows of $c$ from $M$ (line 5).
Moreover, it iterates over each interacted row $r$ of column $c$ (lines 6 - 10).
We obtain a submatrix $M'$ by covering each interacted column $j$ of row $r$ (line 7).
It constructs an ZBDD node $\beta$ representing all exact covers for the submatrix $M'$ (line 8).
If $\beta$ represents a non-empty set, then the algorithm updates $\alpha$ by building an ZBDD node with decision variable $r$, negative successor node $\alpha$ and positive successor node $\beta$, representing all exact covers containing $r$ (line 9).
At the end of the iteration, we recover all deleted columns and rows in $M$ at line 7 so as to backtrack on the next interacted row (line 10).
Finally, the algorithm recovers matrix $M$, terminates with the resulting ZBDD $\alpha$ which is added to the memoization cache (lines 11 - 13).

\looseness=-1
Algorithm DXZ utilizes the data structure: dancing links to implement three operations: (1) search of interacted rows and columns of an incidence matrix; (2) covering a column, and (3) uncovering a column.
In a dancing link, each non-zero entry $M_{i, j}$ of a matrix $M$ is represented as a cell $x$ with $5$ fields: $\leftfield$, $\rightfield$, $\upfield$, $\downfield$ and $\headfield$.
$\leftfield[x]$ and $\rightfield[x]$ denote the cell representing the predecessor and successor non-zero entry $M_{i, k}$ of $M$ with the same row $i$ of $x$, respectively.
$\upfield[x]$ and $\downfield[x]$ have the same meaning but for the same column $j$ of $x$.
Each row (resp. column) of $M$ is a circular doubly-linked list via fields $\leftfield$ and $\rightfield$ (resp. $\upfield$ and $\downfield$).
Each column list has a header cell $h$ with an additional field $\sizefield$ where $\sizefield[h]$ denotes the the number of interacted rows of this column in matrix $M$.
In addition, all of header cells form a special row.
$\headfield[x]$ denotes the header cell $h$ of the column list in which $x$ is.

\looseness=-1
The search of interacted rows of column $c$ can be accomplished via starting from cell $x$ where $x$ is a cell in column $c$, and iteratively traversing the predecessor (or successor) non-zero entry via field $\leftfield[x]$ (or $\rightfield[x]$) until we reach cell $x$.
The search of interacted columns of row $r$ can be similarly implemented by using fields $\upfield[x]$ and $\downfield[x]$.

\looseness=-1
The pseudo-code of covering and uncovering a column are illustrated in Algorithm \ref{alg:cover} and \ref{alg:uncover}, respectively.
Suppose that cell $x$ represents a non-zero entry $M_{r, c}$.
Algorithm \ref{alg:cover} uses the two operations: $\leftfield[{\rightfield[x]}] \assign \leftfield[x]$ and $\rightfield[{\leftfield[x]}] \assign \rightfield[x]$
to delete column $c$ from $M$.
In addition, Algorithm \ref{alg:cover'} returns the set of rows that are covered during the execution of {\tt Cover}$(c)$.
In contrast, Algorithm \ref{alg:uncover} uses the two operations: 
$\leftfield[{\rightfield[x]}] \assign x$ and $\rightfield[{\leftfield[x]}] \assign x$
to recover column $c$ in $M$.
The deletion and recovery of row $r$ are the same except we use fields $\upfield$ and $\downfield$ instead of $\leftfield$ and $\rightfield$, respectively.
The running times of Algorithms \ref{alg:cover} and \ref{alg:uncover} are proportional to the number of affected matrix entries, making DLX efficient for sparse matrices.

\begin{figure*}[t]
    \centering
	\makebox[\textwidth][c]{
    \begin{subfigure}{0.32\textwidth}
        \centering
		{\footnotesize
		\[
        \begin{array}{c@{\;}c}
        &
        \begin{array}{cccccc}
        1 & 2 & 3 & 4 & 5& 6
        \end{array}\\
        \begin{array}{c}
        A \\ B \\ C \\ D \\ E \\ F
        \end{array}
        &
        \left(
        \begin{array}{cccccc}
        1&1&1&1&0&0\\
        1&0&0&1&0&0\\
        0&1&1&0&0&0\\
        0&0&0&0&1&1\\
        0&0&0&0&0&1\\
        0&0&0&0&1&0
        \end{array}
        \right)
        \end{array}
        \]
		}
        \caption{An incident matrix $M$.}
    \end{subfigure}
    \hspace{0.06\textwidth}
    \begin{subfigure}{0.26\textwidth}
        \centering
		{\footnotesize		
		\[
        \begin{array}{c@{\;}c}
        &
        \begin{array}{cc}
        5&6
        \end{array}\\
        \begin{array}{c}
        D\\E\\F
        \end{array}
        &
        \left(
        \begin{array}{cc}
        1&1\\
        0&1\\
        1&0
        \end{array}
        \right)
        \end{array}
        \]
		}
        \caption{A submatrix of $M$.}
    \end{subfigure}
    \hspace{0.06\textwidth}
    \begin{subfigure}{0.26\textwidth}
        \centering
        \includegraphics[width=1.60cm]{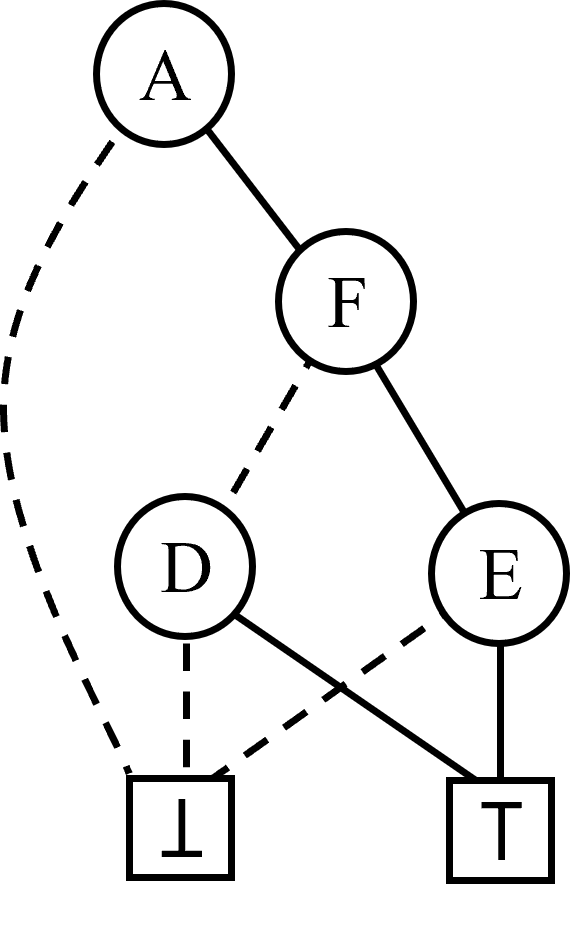}
        \caption{An intermediate ZBDD.}
    \end{subfigure}
	}

    \medskip

	\makebox[\textwidth][c]{
    \begin{subfigure}{0.26\textwidth}
        \centering
		{\footnotesize
		\[
        \begin{array}{c@{\;}c}
        &
        \begin{array}{cccc}
        2&3&5&6
        \end{array}\\
        \begin{array}{c}
        C\\D\\E\\F
        \end{array}
        &
        \left(
        \begin{array}{cccc}
        1&1&0&0\\
        0&0&1&1\\
        0&0&0&1\\
        0&0&1&0
        \end{array}
        \right)
        \end{array}
        \]
		}
        \caption{A submatrix of $M$.}
    \end{subfigure}
    \hspace{0.06\textwidth}
    \begin{subfigure}{0.20\textwidth}
        \centering
        \includegraphics[width=1.60cm]{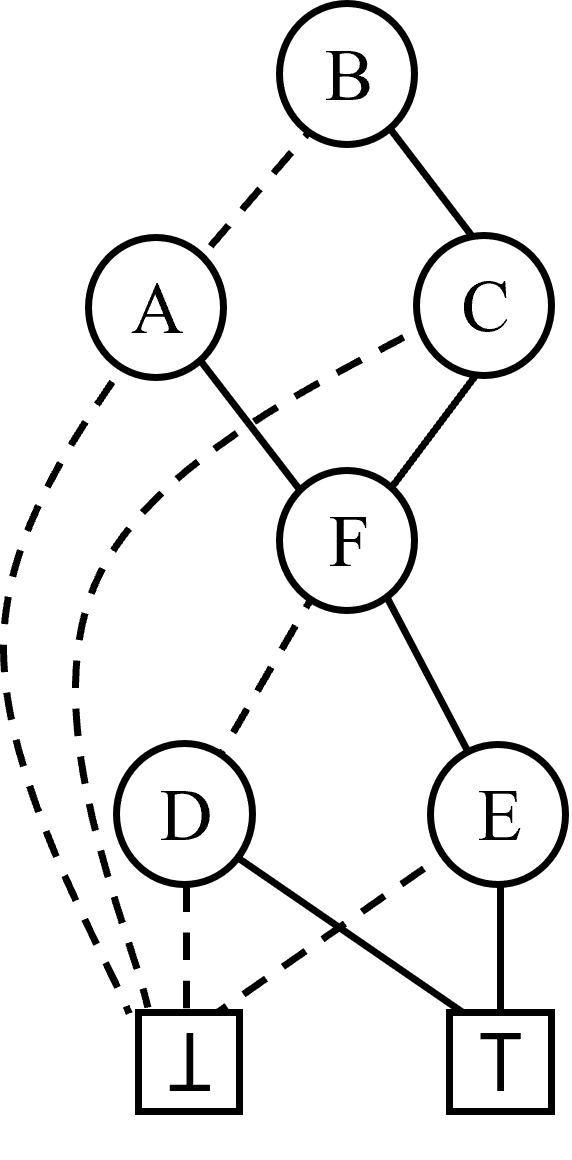}
        \caption{The final ZBDD.}
    \end{subfigure}
	}

    \caption{A run of Algorithm DXZ}
    \label{fig:DLX_DXZ}
\end{figure*}

\begin{example}
	Consider the incidence matrix in Fig. \ref{fig:DLX_DXZ} (a) as input $M$ to Algorithm DXZ.  
	The algorithm first selects Column $1$, whose interacted rows are $A$ and $B$.
	
	For Row $A$, the algorithm covers Columns $1 - 4$ and removes their interacted rows $A$, $B$  and $C$, leaving a submatrix with Columns $5$ and $6$ and Rows $D$, $E$ and $F$, shown in Fig. \ref{fig:DLX_DXZ} (b).
	Then the algorithm recursively runs on the submatrix.
	It chooses Column $5$ since it has the minimal number of interacted rows.
	Rows $D$ and $F$ are the two interacted rows of Column $5$.
	Eliminating Row $D$ and its interacted columns yields the empty matrix, and an exact cover $\set{A, D}$.
	Similarly, removing Row $F$ and its interacted columns generates the submatrix with only Column $6$ and Row $E$, which corresponds to another exact cover $\set{A, E, F}$.
	The ZBDD representing the above two solutions are shown in Fig. \ref{fig:DLX_DXZ}(c).

	Now, the algorithm handles the case: Row $B$.
	Covering $B$ removes columns $1$ and $4$ and their interacted rows, resulting in the submatrix shown in Fig. \ref{fig:DLX_DXZ}(d).
	Since Column $2$ has the minimum number of interacted rows, the algorithm covers this column and obtains the same submatrix with Columns $\set{5, 6}$ and Rows $\set{D, E, F}$ as in the case of selecting Row $A$.
	The algorithm directly retrieves the previously constructed ZBDD node from the cache.
	Hence, we get two exact covers: $\set{B, C, D}$ and $\set{B, C, E, F}$.
	Combining both two cases, Algorithm DXZ returns the final ZBDD shown in Fig.~\ref{fig:DLX_DXZ}(e).\qed
\end{example}

\section{Algorithm DXD}

\looseness=-1
In this section, we first provide the syntax and semantics of zero-suppressed decision-DNNF that serves as the theoretical foundation for representing families of sets.
Then, we develop a novel algorithm, namely DXD, by integrating the process of constructing decomposable nodes into Algorithm DXZ based on the concept of connected components.
Furthermore, we design an improvement to DXD, namely DynDXD, which dynamically updates connected components. 
Finally, we analyze the complexity of the three algorithms: DXZ, DXD and DynDXD.

\subsection{Zero-Suppressed Decision-DNNF}
\begin{definition}\label{def:DZDNNF} \rm
	A formula $\alpha$ is in zero-suppressed decision-DNNF, iff 
	\begin{itemize}
		\item $\alpha$ is $\true$, $\false$ or $v$ where $v$ is a variable; \\
		Semantics: $\semantics{\true} = \set{\emptyset}$, $\semantics{\false} = \emptyset$ and $\semantics{v} = \set{\set{v}}$;
		
		\item $\alpha$ is a decision node that is a tuple $(v, \beta_1, \beta_2)$ s.t. $v \notin \varSet(\beta_1)$ and $v \notin \varSet(\beta_2)$ where $V(\beta_i)$ denotes the variables occurring in the node $\beta_i$; \\
		Semantics: $\semantics{(v, \beta_1, \beta_2)} = (\set{\set{v}} \orthJoin \semantics{\beta_1})\cup \semantics{\beta_2}$;
		
		\item $\alpha$ is a decomposable node that contains the set of nodes $\set{\beta_1, \cdots, \beta_n}$ s.t. $\varSet(\beta_i) \cap \varSet(\beta_j) = \emptyset$ for $i \neq j$; \\		
		Semantics: $\set{\beta_1, \cdots, \beta_n} = \semantics{\beta_1} \orthJoin \cdots \orthJoin \semantics{\beta_n}$.
	\end{itemize}	
\end{definition}

\looseness=-1
Given a decision node $\alpha = (v, \beta_1, \beta_2)$, we say $v$ is the decision variable of $\alpha$, $\beta_1$ is the positive formula and $\beta_2$ is the negative formula. 
Decision-DNNFs are interpreted as Boolean functions and the zero-suppressed variant as families of sets. 
Any negative literal $\neg v$ represents the same family of sets $\set{\emptyset}$ as $\true$.
Therefore, the syntax of decision-ZDNNFs does not contain negative literals.
The decision node $(v, \beta_1, \beta_2)$ of a zero-suppressed decision-DNNF formula can be equivalently simplified to $\beta_2$ when $\semantics{\beta_1} = \emptyset$.
In contrast, the above simplification is valid (that is, the decision node $(v, \beta_1, \beta_2)$ and the sub-node $\beta_2$ represent the same Boolean function) in decision-DNNFs when $\beta_1$ and $\beta_2$ represent the same Boolean function.

\begin{theorem}\label{thm:succinct}
Decision-ZDNNFs are exponentially more succinct than ZBDDs.
\end{theorem}

\begin{proof}
	In BDDs, there are three reduction rules:
	\begin{itemize}
		\item Standard reduction rule: if a decision node $\alpha$ has two identical successors, then remove the node $\alpha$ and redirect all incoming edges of $\alpha$ to its successor.
		
		\item Zero-suppressed reduction rule: if the positive successor of a decision node $\alpha$ is the terminal node $\false$, then remove the node $\alpha$ and redirect all incoming edges of $\alpha$ to its successor.
		
		\item Isomorphic reduction rule: if two decision nodes $\alpha$ and $\alpha'$ have the same decision variable $v$ and two successors $\beta_1$ and $\beta_2$, then remove the node $\alpha'$ and redirect all incoming edges of $\alpha'$ to $\alpha$.
	\end{itemize}
	
	Traditional reduced BDDs (RBDDs) \cite{Bry1986} allow standard and isomorphic reduction rules to eliminate nodes of a BDD.
	For a Boolean function, a BDD has the minimal size iff none of standard and isomorphic reduction rules applies on it.	
	ZBDDs \cite{Min1993} use zero-suppressed reduction rules instead of standard one. 
	Similarly, a ZBDD has the minimal size iff none of zero-suppressed and isomorphic reduction rules applies on it.	
	Quasi-reduced BDDs (QBDDs) \cite{KimC1990} allow only isomorphic reduction rule.
	By recovering the nodes eliminated by the standard (resp. zero-suppressed) reduction rule, a RBDD (resp. ZBDD) $\alpha$ can be converted into an equivalent QBDD with at most $n \times \size{\alpha}$ where $n$ is the number of variables and $\size{\alpha}$ is the number of nodes in $\alpha$ \cite{Weg2000,Knu2016}.
	As a corollary, RBDDs and ZBDDs are equally succinct.	
	Similarly, decision-DNNF and decision-ZDNNFs are equally succinct.
	In addition, there are a class of Boolean functions with polynomial-size decision-DNNF representations \cite{Weg1987} but exponential-size RBDDs \cite{DarM2002,BeaLRS2017}.
	Hence, decision-DNNF are exponentially more succinct than RBDDs.	
	In summary, we obtain that decision-ZDNNF are exponentially more succinct than ZBDDs.	
\end{proof}

\begin{algorithm}[t]\small
\caption{${\tt DXD}(M, \extendedGraphSet)$}
\label{alg:dxd}

\KwIn{$M$: an incidence matrix $\extendedGraphSet$: the set $\set{G_i = (V_i, E_i)}_{1 \leq i \leq n}$ of connected components of the primal graph of $M$}
\KwOut{$\alpha$: a root decision-ZDNNF node representing all exact cover solutions for $M$}

  \lIf{$M$ is empty}{
      \Return{$\true$}
  }

  \lIf{$M$ has only one row $r$ that interacting all columns}{
      \Return{$r$}
  }

  \lIf{${\tt Cache}({\tt Col}(M)) \neq nil$}
  {
      \Return{${\tt Cache}({\tt Col}(M))$}
  }

  \If{$|\extendedGraphSet| \geq 2$}
  {
  	$M_1 \cdots, M_n \assign {\tt DecomposeMatrix}(M, \extendedGraphSet)$
  	
      \lFor{$i = 1$ \KwTo $n$} {
          $\beta_i \assign {\tt DXD}(M_i, \set{G_i})$
      }
      
      $\alpha \assign {\tt BuildDecomposableNode}(\beta_1, \cdots, \beta_n)$
     
      ${\tt Cache}({\tt Col}(M)) \assign \alpha$

      \Return{$\alpha$}
      \vspace*{-1mm}
  } 
  
  $c \assign $ a column of $M$ with minimal number of interacted rows
  
  $\alpha \assign \false$

  $V^d_c \assign {\tt Cover'}(c)$
  
  $E^d_c \assign$ the set of edges incident to any vertex of $V^d_c$ in $\extendedGraphSet$

  ${\tt DecUpdateCC}(\extendedGraphSet, V^d_c, E^d_c)$

  \ForEach{interacted row $r$ of column $c$}
  {
      \ForEach{interacted column $j$ of row $r$}
      {
      		$V^d_j \assign {\tt Cover'}(j)$
      		
      		$E^d_j \assign$ the set of edges incident to any vertex of $V^d_j$ 
      		
      		${\tt DecUpdateCC}(\extendedGraphSet, V^d_j, E^d_j)$
      		\vspace*{-1mm}
      }

      ${\tt \beta} \assign {\tt DXD}(M, \extendedGraphSet)$

      \lIf{${\tt \beta} \ne \false$}{
          $\alpha \assign {\tt BuildDecisionNode}(r, \alpha, \beta)$
      }

      \ForEach{interacted column $j$ of row $r$}
      {
        ${\tt Uncover}(j)$
        
        ${\tt IncUpdateCC}(\extendedGraphSet, V^d_j, E^d_j)$
        \vspace*{-1mm}
      }\vspace*{-1mm}
  }

  ${\tt Uncover}(c)$

  ${\tt IncUpdateCC}(\extendedGraphSet, V^d_c, E^d_c)$

  ${\tt Cache}({\tt Col}(M)) \assign \alpha$

  \Return{$\tt \alpha$}
\vspace*{-1mm}
\end{algorithm}

\subsection{The Main Algorithm}
\looseness=-1
To build decomposable nodes, we make use of the concept of primal graphs.
Given an incidence matrix $M$, the primal graph $G = (V, E)$ of $M$ is defined as: (1) the vertex set $V$ corresponds to the rows of $M$; and (2) an edge $(r_1,r_2)\in E$ exists if and only if rows $r_1$ and $r_2$ share at least one common interacted column in $M$.
We compute the set $\extendedGraphSet: \set{G_i = (V_i, E_i) \mid 1 \leq i \leq n}$ of connected components of $G$.
These disjoint subsets $V_1, \cdots, V_n$ form a partition of the row set $V$.
Rows belonging to the same component are pairwise connected through shared columns, while rows from different components share no common interacted columns.
Each row subset $V_i$ thus induces an independent submatrix $M_i$ of $M$, consisting of rows in $V_i$ and the columns interacting with them.
All connected components can be generated in linear time in the size of the graph using breadth-first search \cite{HopT1973}.
Suppose that $M$ is decomposed into submatrices $M_1, \cdots, M_n$ in this manner.
Then the set $\exactCoverSolutions$ of exact covers for $M$ factorizes as the orthogonal join of the solution sets $\exactCoverSolutions_i$ of the submatrices $M_i$, \ie, $\exactCoverSolutions = \exactCoverSolutions_1 \orthJoin \cdots \orthJoin \exactCoverSolutions_n$.
This factorization directly corresponds to a decomposable node in the resulting decision-ZDNNF.
Moreover, each set $\exactCoverSolutions_i$ can be computed in parallel since all submatrices are independent.

\looseness=-1
Following the above idea, we develop a modified algorithm DXD, illustrated in Algorithm \ref{alg:dxd}, which integrates the process of building decomposable nodes into Algorithm DXZ. 
It takes an incidence matrix $M$ and the set $\extendedGraphSet$ of connected components of the primal graph of $M$ as input. 
It first examines if the incidence matrix $M$ is empty (line 1) or has only one row that covers all columns (line 2).
In either case, it returns the terminal node $\true$ or the row node $r$ as a literal node of decision-ZDNNF, respectively.
If the set $\extendedGraphSet$ has more than one connected component, then the incidence matrix $M$ can be decomposed into several submatrices $M_1, \cdots, M_n$.
Algorithm \ref{alg:dxd} recursively generates each zero-suppressed decision-DNNF node $\beta_i$ representing all exact covers for $M_i$ for $1 \leq i \leq n$ (line 6). 
All nodes $\beta_i$ form a decomposable node $\alpha$ (lines 7).
If $M$ cannot be decomposed, then it will construct a hierarchical decision node via deleting all of interacted rows and columns for the chosen column $c$ as Algorithm DXZ does (lines 10 - 28).
In addition, covering columns not only remove the set $V^d$ of rows but also records them (lines 12 \& 17).
Furthermore, the set $E^d$ of edges incident to $V^d$ is also recorded (lines 13 \& 18).
Then, it updates the set $\extendedGraphSet$ of connected components by deleting both $V^d$ and $E^d$ (lines 14 \& 19). 
These vertices and edges will be later inserted into $\extendedGraphSet$ after the corresponding uncover operations (lines 24 \& 26).

\begin{example}

	\begin{figure}
		\centering
		\includegraphics[width=1.0\textwidth]{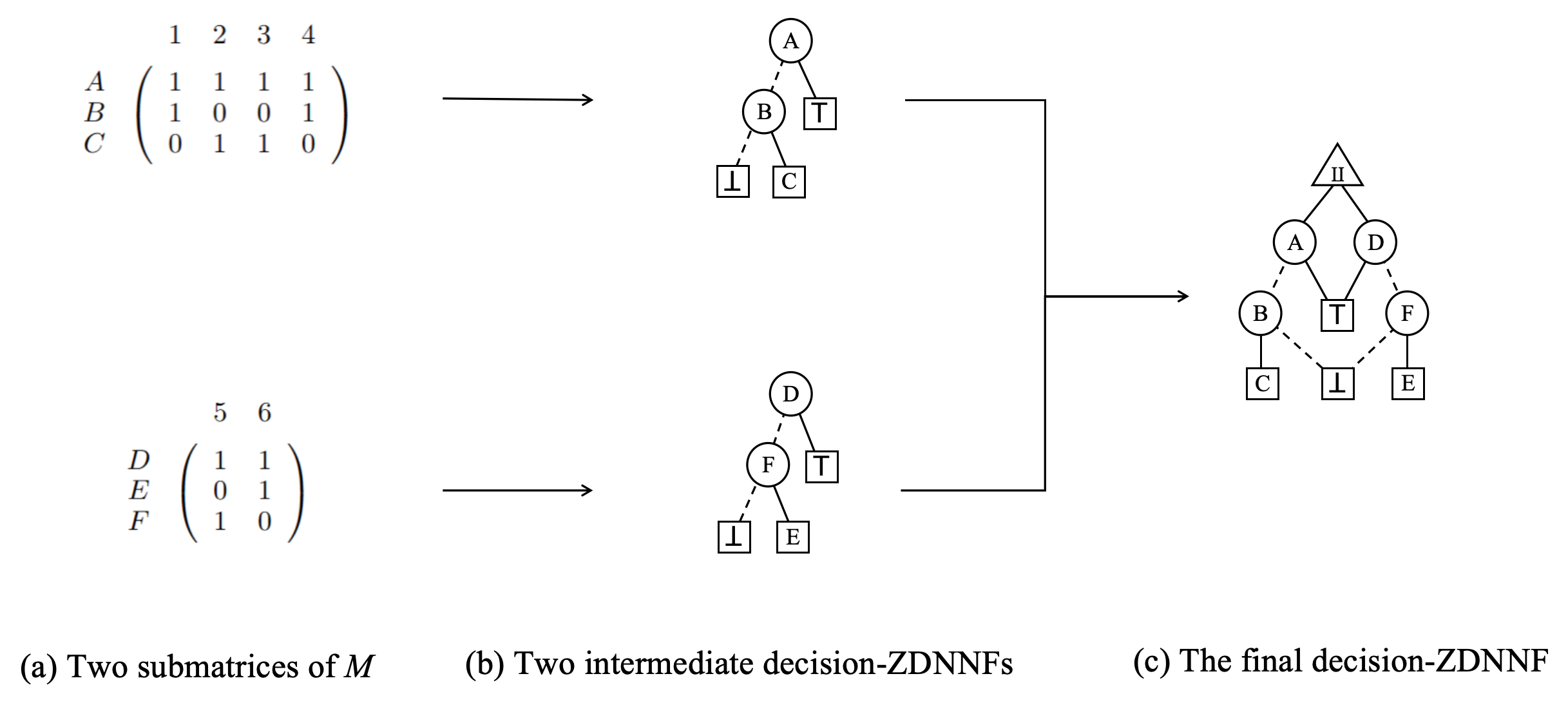}
		    \vspace*{-3mm}
		\caption{A run of Algorithm DXD.}
		    \vspace*{-3mm}
		\label{fig:dxd_dnnf_con}
	\end{figure}
	
\looseness=-1
We introduce how DXD constructs decision-ZDNNFs.
Fig. \ref{fig:dxd_dnnf_con} illustrates the procedure of constructing the decision-ZDNNF representing the set of all exact covers of the matrix shown in Fig. \ref{fig:DLX_DXZ}(a).
Initially, the matrix shown in Fig. \ref{fig:DLX_DXZ}(a) can be decomposed into two mutually independent submatrices $M_1$ with rows $\set{A, B, C}$ and $M_2$ with rows $\set{D, E, F}$, shown in Fig. \ref{fig:dxd_dnnf_con}(a).
As depicted in Fig. \ref{fig:dxd_dnnf_con}(b), we obtain decision nodes for matrix $M_1$ and for matrix $M_2$ by iteratively covering and uncovering columns.
The final decision-ZDNNF, representing all of the exact covers of matrix $M$, is obtained by combining the two decision nodes into a decomposition node labeled by a triangle with $\orthJoin$, depicted in Fig. \ref{fig:dxd_dnnf_con}(c). \qed

\end{example}

\begin{algorithm}\small
	\caption{${\tt Cover'}(c)$}
	\label{alg:cover'}
	\KwIn{$c$: a column}
    \KwOut{$V^{d}_c$: the set of covered rows}
	
	$h \assign$ the header cell for column $c$
	
	$V^{d}_c \assign \emptyset$
	
	$\leftfield[{\rightfield[h]}] \assign \leftfield[h]$ 
	
	$\rightfield[{\leftfield[h]}] \assign \rightfield[h]$
	
	$i \assign \downfield[h]$
	
	\While{$i \neq h$}{
		
		$j \assign \rightfield[i]$
		
		\While{$j \neq i$}
		{
			$\upfield[{\downfield[j]}] \assign \upfield[j]$
			
			$\downfield[{\upfield[j]}] \assign \downfield[j]$
			
			$\sizefield[{\headfield[j]}] \assign \sizefield[{\headfield[j]}] - 1$
			
			$j \assign \rightfield[j]$
			
		}

        $V^{d}_c \assign V^{d}_c \cup \set{\rowfield[i]}$
		
		$i \assign \downfield[i]$
		
	}

    \Return{$V^{d}_c$}
\end{algorithm}

\begin{algorithm}[t]\small
\caption{${\tt DecUpdateCC}(\extendedGraphSet, V^{d}, E^{d})$}
\label{alg:decgencc}
\KwIn{$\extendedGraphSet$: the set of connected components of the graph $G$ \\
	  \hspace*{8.5mm} $V^{d}$: the set of vertices of $G$ to be deleted \\
	  \hspace*{8.5mm} $E^{d}$: the set of edges of $G$ to be deleted
}
    \ForEach{$(u, v) \in E^{d}$}
    {
    	$(T, E^{n}) \assign {\tt FindCC}(u)$
    	
        \lIf{$(u, v) \in E^{n}$}
        {
        	$E^{n} \assign E^{n} \setminus \set{(u, v)}$
        }
        \Else
        {
			$Found = \textbf{false}$
			
            $T_u, T_v \assign {\tt Cut}(u, v)$

            \lIf{$\sizefield[T_v] > \sizefield[T_u]$}{
                Swap $T_u$ and $T_v$
            }

            \ForEach{$(x, y) \in E^{n}$}
            {
            	\If{exactly one of $x$ and $y$ is in $T_u$}
            	{
            		$T \assign {\tt Link}(x, y)$
            		
            		$E^{n} \assign E^{n} \setminus \set{(x, y)}$
            		
            		$Found = \textbf{true}$
            		
            		\textbf{break}\vspace*{-1mm}           	
            	}
            	\lElseIf{$x, y \notin T_v$}
            	{
            		$E^{n}_{u} \assign E^{n}_{u} \cup \set{(x, y)}$ 
            	}
            	\lElse(\tcc*[f]{$x, y \in T_v$})
            	{
            			$E^{n}_{v} \assign E^{n}_{v} \cup \set{(x, y)}$
            	}
            }

            \If{$\neg Found$}
            {
            	$\extendedGraphSet \assign (\extendedGraphSet \setminus \set{(T, E^{n})}) \cup \set{(T_u, E_u^{n}), (T_v, E_v^{n})}$
            }\vspace*{-1mm}
        }\vspace*{-1mm}
    }
    
    \lForEach{$v \in V^d$}   	
    {
    	$\extendedGraphSet \assign \extendedGraphSet \setminus \set{(\set{v}, \emptyset), \emptyset)}$
    }

\vspace*{-1mm}
\end{algorithm}

\subsection{Dynamically Updating Connected Components}

\looseness=-1
In Algorithm DXD, one critical step is to generate connected components of the primal graph under vertices and edges deletions and insertions.
After covering several columns from $M$, the edges incident to any one vertex corresponding to a covered column are removed from the primal graph. 
Subsequently, the new set of connected components will be recomputed from scratch.
Similarly, Algorithm DXD performs recomputation after uncovering operations.
This motivates us to develop a more efficient approach to update components in a dynamic manner.

\looseness=-1
We first introduce several concepts in graph theory.
Let $G = (V, E)$ be a connected undirected graph.
A \textit{spanning tree} $G' = (V, E')$ of $G$ is a subgraph of $G$ s.t. $|E'| = |V| - 1$ and all vertices of $V$ are still connected.
The set $E^n$ of \textit{non-tree edges} of $G$ relative to a spanning tree $T$ is defined as $E \setminus E'$.
A connected component $G$ can be represented by a pair $(T, E^n)$ where $T$ is a spanning tree $T$ and $E^n$ is the set of non-tree edges relative to $T$.

\looseness=-1
We hereafter introduce several basic operations on connected components and spanning trees.
$\tt{FindCC}(v)$ locates the connected component $(T, E^n)$ with vertex $v$. 
When two vertices $u$ and $v$ are in two different spanning trees $T_u$ and $T_v$, $\tt{Link}(u, v)$ merges $T_u$ and $T_v$ by connecting $u$ and $v$. 
When $u$ and $v$ are in the same spanning tree $T$, $\tt{Cut}(u, v)$ removes edge $(u, v)$ from $T$, and returns two spanning trees $T_u$ with $u$ and $T_u$ with $v$, respectively.

\looseness=-1
Algorithm \ref{alg:decgencc} illustrates a decremental approach to updating connected components. 
It takes a set of connected components $\extendedGraphSet$ of a graph $G$ and a set $E^{d}$ of edges of $G$ as input, and updates $\extendedGraphSet$ by removing $E^d$ from $G$.
The algorithm visits each edge $(u, v) \in E^{d}$ sequentially (lines 1 - 17).
It first obtains the connected component $(T, E^n)$ in which $u$ is (line 2).
If an edge $(u, v)$ is a non-tree edge, it simply removes it from $E^{n}$ (line 3).
This does not affect the connectivity of any two vertices.
Otherwise, $(u, v)$ is an edge in $T$.
In this case, the algorithm splits the spanning tree $T$ into two subtrees $T_u$ and $T_v$ (line 6).
To improve the efficiency of subsequent operations, we let $T_v$ be the spanning tree with smaller size (line 7).
Then the algorithm attempts to search for a non-tree edge so as to connect $T_u$ and $T_v$ (lines 8 - 15).
If there is a non-tree edge $(x, y)$ s.t. exactly one of $x$ and $y$ is in $T_v$, then $(x, y)$ is the desired edge (line 9).
The algorithm reconnects the two trees $T_u$ and $T_v$ as a new spanning tree $T$, and removes $(x, y)$ from $E^n$ (lines 10 - 12).
Otherwise, we construct two sets of non-tree edges $E^n_u$ and $E^n_v$ during the search (lines 14 \& 15).
In this case, $(T_u, E^n_u)$ and $(T_v, E^n_v)$ are two different connected components with vertices $u$ and $v$, respectively, and the original component $(T, E^n)$ is replaced by $(T_u, E^n_u)$ and $(T_v, E^n_v)$ (lines 16 \& 17).
After deleting all edges incident to $V^d$, the connected component with any vertex $v \in V^d$ is a single vertex with no edge, that is, $((\set{v}, \emptyset), \emptyset)$.
Finally, we delete these components from $\extendedGraphSet$ (line 18).

\looseness=-1
Algorithm \ref{alg:incgencc} elaborates an incremental approach to updating connected components under edge insertions, which is a reverse process to Algorithm \ref{alg:decgencc}.
The algorithm first adds all vertices $v \in V^d$ into the graph $G$, which adds a $((\set{v}, \emptyset), \emptyset)$ into $\extendedGraphSet$ (line 1).
For each edge $(u, v) \in E^{d}$, the algorithm then identifies the two connected components $(T_u, E^{n}_{u})$ and $(T_v, E^{n}_{v})$ in which $u$ and $v$ are, respectively (lines 3 \& 4).
If $u$ and $v$ are in different spanning trees, then the two trees $T_u$ and $T_v$ are merged as a new one $T$ by adding $(u, v)$ and the two sets $E^{n}_{u}$ and $E^{n}_{v}$ of non-tree edges are combined as $E^n$ (lines 6 \& 7).
The algorithm then removes the two components $(T_u, E^{n}_{u})$ and $(T_v, E^{n}_{v})$ from $\extendedGraphSet$, and adds the new one $(T, E^n)$ (line 8).
Otherwise, $u$ and $v$ are connected before inserting edge $(u, v)$.
So we simply add edge $(u, v)$ into $E^{n}_{u}$ (line 9).

\begin{algorithm}[t]\small
	\caption{${\tt IncUpdateCC}(\extendedGraphSet, V^{d}, E^{d})$}
	\label{alg:incgencc}
	\KwIn{$\extendedGraphSet$: the set of connected components of the graph $G$ \\		
		\hspace*{8.5mm} $V^{d}$: the set of vertices of $G$ to be deleted \\
		\hspace*{8.5mm} $E^{d}$: the set of edges of $G$ to be deleted
	}

\lForEach{$v \in V^d$}
{
	$\extendedGraphSet \assign \extendedGraphSet \cup \set{(\set{v}, \emptyset)}$
}

\ForEach{$(u, v) \in E^{d}$}
{
	$(T_u, E^{n}_u) \assign {\tt FindCC}(u)$ 
	
	$(T_v, E^{n}_v) \assign {\tt FindCC}(v)$ 
	
	\If{$T_u \neq T_v$}
	{
		$T \assign {\tt Link}(u, v)$ 
		
		$E^{n} \assign E^{n}_u \cup E^{n}_v$
		
		$\extendedGraphSet \assign (\extendedGraphSet \setminus \set{(T_u, E^{n}_u), (T_v, E^{n}_v)}) \cup \set{(T, E^{n})}$
	}
	\lElse
	{
		$E^{n}_u \assign E^{n}_u \cup \set{(u, v)}$		
	}
}\vspace*{-1mm}
\end{algorithm}

\begin{example}

\looseness=-1
Fig. \ref{fig:Dyn_CC}(a) depicts a connected undirected graph $G = (V, E)$. In this graph,  
$V=\{v_1,v_2,v_3,v_4,v_5\}$ and 
$E = \{e_1, e_2, e_3, e_4, e_5, e_6\}$.
A spanning tree $T$ of $G$ is $(V, E')$ where $E' = \set{e_1, e_3, e_5, e_6}$, denoted by solid lines.
The set of non-tree edges relative to $T$ is $E^n = \set{e_2, e_4}$, denoted by dashed lines.
Since $G$ has only one connected component $(T, E^n)$, the set $\extendedGraphSet$ of connected components is $\set{(T, E^n)}$.

\looseness=-1
The algorithm deletes vertex $v_1$ and its associated edges $\set{e_1, e_2, e_3}$ from $G$ sequentially.
Since edge $e_1$ is not a non-tree edge, the spanning tree $T$ is split into two subtrees $T_{v_1}$ and $T_{v_2}$, shown in Fig. \ref{fig:Dyn_CC}(b).
A replacement edge $e_4$ in $E^n$ is found.
We update the spanning tree $T$ by connecting $T_{v_1}$ and $T_{v_2}$ via $e_4$, and the set of non-tree edges $E^n: \set{e_2}$.
Then, the algorithm directly removes the edge $e_2$ from $E^n$ since it is a non-tree edge (Fig.~\ref{fig:Dyn_CC}(c)).
So $E^n$ becomes an empty set.
For edge $e_3$, no edges exist in $E^n$.
Consequently, the graph is decomposed into two connected components with spanning trees $T'_{v_1}$ and $T'_{v_2}$ and empty sets of non-tree edges (Fig.~\ref{fig:Dyn_CC}(d)).
Finally, vertex $v_1$ is removed.
So is $(T'_{v_1}, \emptyset)$.
We obtain the final graph $\extendedGraphSet = \set{(T'_{v_2}, \emptyset)}$, shown in Fig. \ref{fig:Dyn_CC}(e).

\looseness=-1
We now illustrate how vertex $v_1$ and its associated edges $\set{e_1, e_2, e_3}$ are inserted into the graph.
First, we add a connected component $(T'_{v_1}, \emptyset)$ with only vertex $v_1$ into $\extendedGraphSet$ where $T'_{1} = (\set{v_1}, \emptyset)$.
Now $\extendedGraphSet$ is $\set{(T'_{v_1}, \emptyset), (T'_{v_2}, \emptyset)}$, shown in Fig. \ref{fig:Dyn_CC}(f).
Then, we add edge $e_1: (v_1, v_2)$.
Since vertices $v_1$ and $v_2$ are in two different trees $T'_{v_1}$ and $T'_{v_2}$, we acquire a new spanning tree $T_*$ by merging  $T'_{v_1}$ and $T'_{v_2}$ via $e_1$, shown in Fig. \ref{fig:Dyn_CC}(g).
For edges $e_2$ and $e_3$, both endpoints are contained in $T_*$, and thus they are added to the set $E^n_*$ of non-tree edges, shown in Fig. \ref{fig:Dyn_CC}(h).

Although $T \neq T_*$ and $E^n \neq E^n_*$, both $(T, E^n)$ and $(T_*, E^n_*)$ represent the same connected components. \qed
\end{example}

\begin{figure}[t]
\centering
\includegraphics[width=1.0\textwidth]{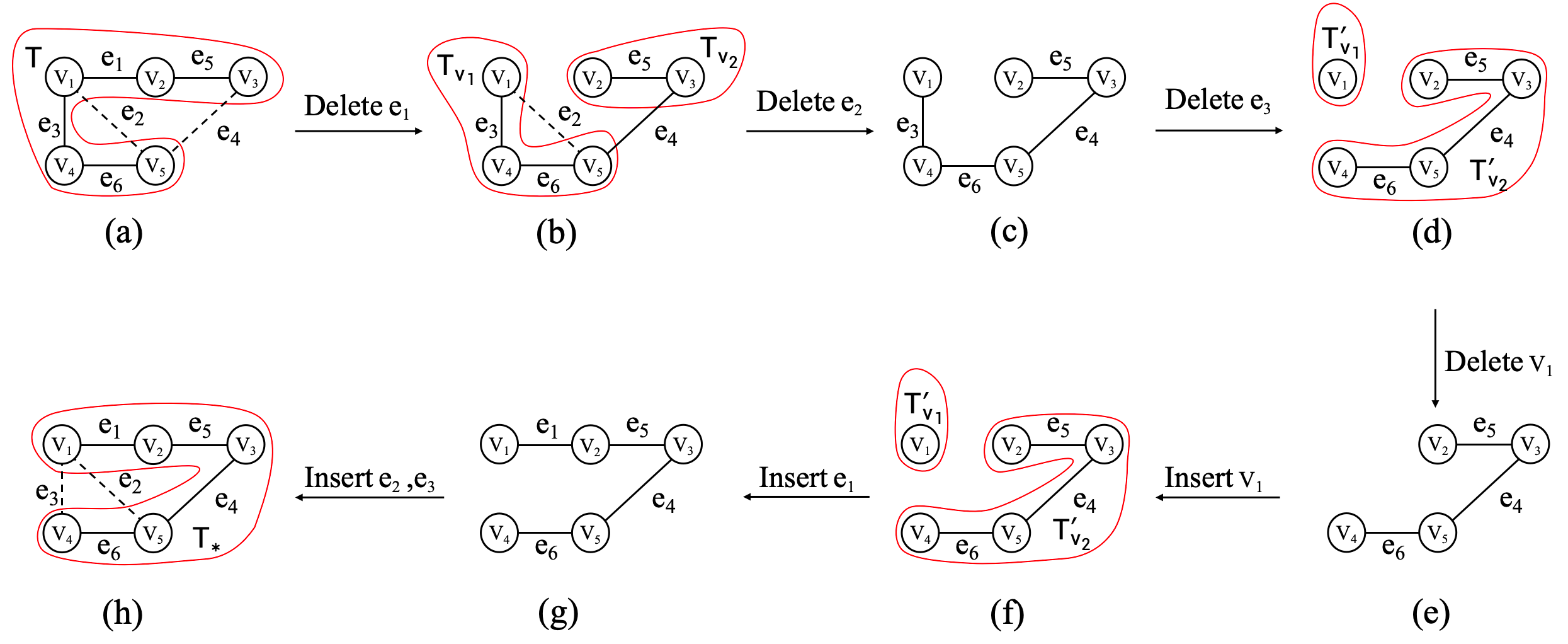}
    \vspace*{-3mm}
\caption{An example of dynamically updating connected components.}
    \vspace*{-3mm}
\label{fig:Dyn_CC}
\end{figure}

\looseness=-1
To implement Algorithms \ref{alg:decgencc} and \ref{alg:incgencc}, we adopt the following data structures to maintain connected components.
A spanning tree can be represented by an Euler tour, which is a linear order of vertices corresponding to a traversal of the tree in which each edge is visited exactly twice.
For each spanning tree $T$, we use a Splay tree \cite{SleT1985} to represent the Euler tour of $T$, together with two hash sets to store the two sets of vertices of $T$ and of non-tree edges relative to $T$.
We also use a hash table that maps each vertex $v$ to its corresponding node in the Splay tree.
With the above data structures, the operations {\tt FindCC}, {\tt Link} and {\tt Cut} can be accomplished in $O(\log |V|)$ where $|V|$ is the number of vertices in the graph $G$.

\begin{theorem}\label{thm:dynamic_complexity}
	Let $G$ be a graph $(V, E)$ with the set $\extendedGraphSet$ of connected components.	
	\begin{itemize}
		\item The time complexity of Algorithm \ref{alg:decgencc} is $O(|E^{d}| (\log |V| + |E|) + |V|)$.
		\item The time complexity of Algorithm \ref{alg:incgencc} is $O(|E^{d}|  \log |V| + |V|)$.
	\end{itemize}
\end{theorem}

\begin{proof}


\looseness=-1
We first analyze the time complexity of Algorithm \ref{alg:decgencc}.
Algorithm \ref{alg:decgencc} traverse each edge $e: (u, v) \in E^{d}$. 
For each edge, the algorithm first invoke the procedure ${\tt FindCC}$ to find the connected component containing $e$, which takes $O(\log |V|)$ time.
If $e$ is a non-tree edge, then it is directly removed from $E^n$.
Otherwise, the algorithm invokes the operation ${\tt Cut}$, which splits the current Euler tour into two Euler tour $T_u$ and $T_v$, accomplished in $O(\log |V|)$ time.
Then, the algorithm searches for a non-tree edge that reconnects $T_u$ and $T_v$.
In the worst case, this search requires $O(|E|)$ time.
Once a replacement edge is found, the algorithm performs the operation ${\tt Link}$ to merge the two trees into a single tree, which takes $O(\log |V|)$ time.
If $e$ is a non-tree edge
The above three operations ${\tt FindCC}$, ${\tt Cut}$ and ${\tt Link}$ are executed at most once for each edge $e$.
Finally, the algorithm rebuilds the new connected components from the graph if some connected components splits during edges deletion.
This takes $O(|V|)$ time.
The time complexity of the algorithm is $O(|E^d| \cdot (\log |V| + |E|) + |V|)$.

\looseness=-1
We then analyze the time complexity of Algorithm \ref{alg:incgencc}.
Similarly to Algorithm \ref{alg:decgencc}, Algorithm \ref{alg:incgencc} traverse each edge $e: (u, v) \in E^{d}$.
For each edge , the algorithm calls two ${\tt FindCC}$ operations to find the connected components $(T_u, E^n_u)$ and $(T_v, E^n_v)$ containing $u$ and $v$, each of which takes $\log |V|$ time.
If $u$ and $v$ are in two different connected components, one operation ${\tt Link}$ is performed to merge the two spanning trees $T_u$ and $T_v$, taking $O(\log |V|)$ time.
Otherwise, $u$ and $v$ are in the same connected component, $(u, v)$ is simply inserted into $E^n$, which takes a constant time.
The algorithm also need to recompute the new connected components if some connected components merge during edged deletion. 
This takes $O(|V|)$ time.
The time complexity of Algorithm \ref{alg:incgencc} is $O(|E^{d}| \cdot \log |V| + |V|)$.
\end{proof}

\looseness=-1
Henzinger~\cite{HenK1999} proposed algorithms to update spanning trees and the set of non-tree edges.
The main difference between their algorithms and ours is the following:
(1) They consider only single edge deletion and insertion while our approach allows the deletion or insertion of multiple vertices and their associated edges;
(2) They only support connectivity queries in polylogarithmic time.
However, obtaining the set $V$ of connected vertices in an updated spanning tree requires an additional $O(|V|)$ time per edge insertion or deletion.
In contrast, we explicitly maintain the set $V$ so as to directly retrieve all connected vertices when original connected components neither split nor merge. 

\subsection{Complexity Analysis}

Let $r$ be the number of rows of the incidence matrix $M$, $c$ the number of columns and $n$ the number of non-zero entries.
Let $z$ and $d$ denote the sizes of the resulting ZBDD and decision-ZDNNF constructed by Algorithms DXZ and DXD, respectively.

\begin{theorem}\label{thm:complexity}
\begin{itemize}
    \item The time complexity of Algorithm DXZ is $O(2^{r} n)$ and the space complexity is $O(n + z)$.
    \item The time complexity of Algorithm DXD is $O(2^{r} (r^{2} + n))$ and the space complexity is $O(n + r^{2} + d)$.
    \item The time complexity of Algorithm DynDXD is $O(2^{r} r^{4})$ and the space complexity is $O(n + r^{2} + d)$.
\end{itemize}
\end{theorem}

\begin{proof}
\looseness=-1
The three algorithms: DXZ, DXD and DynDXD recursively transform a submatrix of the incidence matrix $M$ into a ZBDD or decision-ZDNNF.
The number of submatrices of $M$ is $2^r$.
	

\looseness=-1
We first analyze the time and space complexity of DXZ.
For each submatrix, DXZ performs a sequence of ${\tt Cover}$ and ${\tt Uncover}$ operations. 
In the worst case, performing the above operations requires $O(n)$ time.
The time complexity of DXZ is $O(2^{r} \cdot n)$.
The space complexity of DXZ consists of a dancing link for $M$ and the unique table and cache table for ZBDDs.
The dancing link takes $O(n)$ space while ZBDDs take $O(z)$.
The space complexity of DXZ is $O(n + z)$.

\looseness=-1
Then, we analyze the time and space complexity of DXD.
In addition to DXZ, DXD performs BFS to compute the connected components of the graph corresponding for each submatrix. 
The time complexity admits an upper bound $O(r + r^2) = O(r^2)$.
The time complexity of DXD is $O(2^{r} \cdot (r^{2} + n))$.
Moreover, DXD maintains a primal graph consisting of $r$ vertices and at most $O(r^2)$ edges.
The unique table and cache for decision-ZDNNF nodes take $O(d)$ space.
The space complexity of DXZ is $O(n + r^2 + d)$.

\looseness=-1
Finally, we analyze the time and space complexity of DynDXD.
According to Theorem \ref{thm:dynamic_complexity}, the time complexity of Algorithm \ref{alg:decgencc} is $O(|E^d| \cdot (\log |V| + |E|) + |V|)$.
Since $|V| < r$, $|E^d| < r^2$ and $|E| < r^2$, the time complexity of Algorithm \ref{alg:decgencc} is  $O(r^4)$.
%
Similarly, the time complexity of Algorithm~\ref{alg:incgencc} is $O(r^2 \log r)$.
As a result, the time complexity of DynDXD is $O(2^{r} \cdot r^{4})$.
In addition, the extra data structures for dynamical updating connected components also take $O(r^2)$.
The space complexity of DynDXD is the same as DXD, which is $O(n + r^2 + d)$.

\end{proof}

\looseness=-1
Both DXZ and DynDXD have exponential-time complexity relative to the number of rows in the worst case.
DXZ has better worst-case time and space complexities than DynDXD in theory because the latter requires an additional step for generating connected components.
However, if the current matrix can be decomposed into multiple independent submatrices during the compilation process, DynDXD counts the exact covers of each submatrix in parallel while DXZ still treats the matrix as a single entity and solves the problem in a sequential manner.
Moreover, decision-ZDNNFs are exponentially more succinct than ZBDDs (Theorem \ref{thm:succinct}).
The efficiency of the algorithm is highly sensitive to the size of the representation of all exact covers.
In summary, DynDXD is an effective practical approach.

\looseness=-1
Suppose that $G = (V, E)$ is a graph, $V^{d}$ is a set of vertices to be deleted and $E^d$ is the set of edges incident to $V^d$.
We observe that the following conditions hold in most cases:
(1) Both $V^d$ and $E^d$ account for only a small proportion of vertices and edges in $G$, respectively; 
(2) No connected component splits (resp. merges) after vertices and edges deletions (resp. insertions);
(3) The maximum number $|E^r|$ of edges searched for finding replacement edges in Algorithm \ref{alg:decgencc} is relatively small.
Under the above conditions, DynDXD maintains the set $\extendedGraphSet$ of connected components in time $O(|E^{d}| \cdot (\log |V| + |E^r|) + |V^{d}|)$ for each deletion and insertion of $V^d$ and $E^d$.
In contrast, DXD requires $O(|V| + |E|)$ time. 
In summary, DynDXD can generate new connected components significantly more efficiently than DXD, which will be demonstrated in the experimental results. 

\begin{figure}[b]
  \centering
  \begin{minipage}[t]{0.48\columnwidth}
    \centering
    \includegraphics[width=1.0\columnwidth]{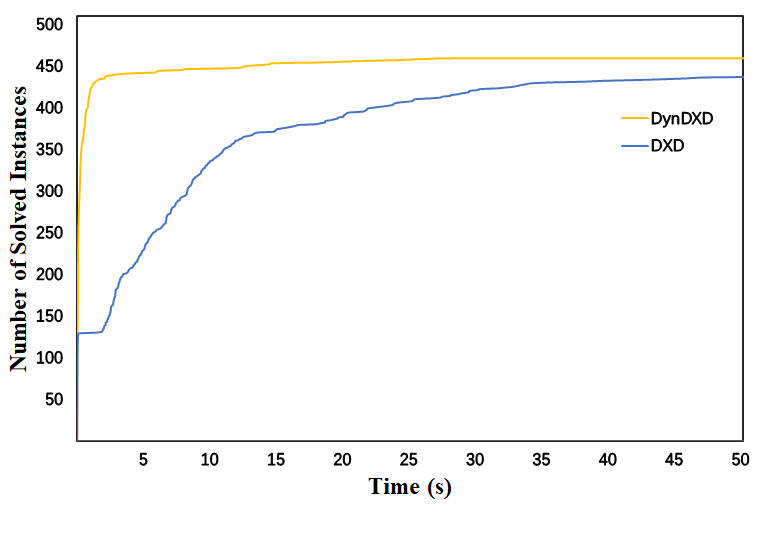}
    \caption{Runtime comparison between DynDXD and DXD.}
    \label{fig:dyndxd_dxd}
  \end{minipage}
  \hfill
  \begin{minipage}[t]{0.48\columnwidth}
    \centering
    \includegraphics[width=1.0\columnwidth]{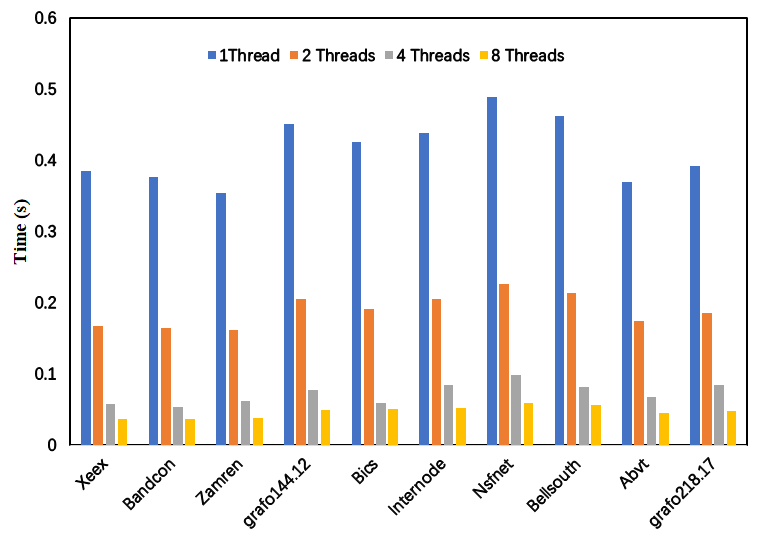}
    \caption{Runtime scalability of DynDXD under different threads.}
    \label{fig:dyndxd_multithread}
  \end{minipage}
\end{figure}

\section{Experiments}

\subsection{Dataset and Settings}

\looseness=-1
We construct exact cover instances from two benchmarks: Internet Topology Zoo \cite{KniNFBR2011} and Rome Graphs \cite{CouMN2014}, also used in Nishino et al.~\cite{NisYN2021}.
We use 245 graphs from Topology Zoo with the number of vertices in the range $[5, 1000]$ and the number of edges in the range $[5, 1500]$. 
We randomly select 255 out of 11,534 graphs from Rome Graphs where the number of vertices and edges of each graph is in the same range as Topology Zoo.
We have 500 graphs in total.
The exact cover instance for a graph is constructed as follows.
We first enumerate all connected components of the graph.
For each component, we select a special vertex and then randomly choose 30\% of the vertices as element vertices.
We enumerate all cycles that contain the special vertex and at least one element vertex.
All element vertices of all components form the universe $\universe$. 
Each cycle corresponds to a subset $S$ of the family $\collection$ where an element $v \in S$ iff the cycle contains element vertex $v$.

\looseness=-1
We compare DynDXD with DLX, DXZ, D3X, and two propositional model counters: SharpSAT-TD \cite{KorJ2021} and ExactMC \cite{LaiMY2021}.
We apply ladder encoding \cite{JunK2010} to translate each exact cover instance into a propositional formula in conjunctive normal form. 
We show the experimental results about the total 500 instances in the following.
All experiments were conducted on a Linux machine with Intel Core i7-10700 2.90GHz CPU and 64GB RAM with a timeout of 1,500 seconds.

\subsection{Results}

\begin{table}[t]
	\centering
	\small
    \begin{tabular}{|c|c|c|c|c|}
        \hline
        \#Thread       & 1 & 2  & 4  & 8  \\ 
        \hline
        Time   & 979.76      & 502.12       & 466.52        & 401.24       \\ 
        \hline
    \end{tabular}
    \vspace*{1mm}
    \caption{\small Total runtime of DynDXD under different threads} 
    \label{tab:dyndxd_multithread}    
\end{table}

\begin{table}[b]
	\setlength\tabcolsep{4pt}
	\centering
	\small
	\begin{tabular}{|c|c|c|c|c|c|c|}
		\hline
		Solvers       & \textbf{DLX} & \textbf{DXZ} & \textbf{D3X} & \textbf{DynDXD} & \textbf{sharpSAT-TD} & \textbf{ExactMC}  \\ 
		\hline
		Instances   & 182      & 460       & 183       & 462     & 425       & 438  \\ 
		\hline
	\end{tabular}
	\vspace*{1mm}
	\caption{\small Total instances solved within the time limit} 
	\label{tab:overall_instances}    
\end{table}

\subsubsection{DynDXD vs DXD}

We first compare the runtime performance of DynDXD and DXD.
As Fig. \ref{fig:dyndxd_dxd} illustrates, DynDXD solves approximately 460 instances within 50 seconds. 
In comparison, DXD solves only slightly more than 440 instances within the same time limit.
DynDXD tends to be much faster than the DXD on every instance.
This improvement is because DynDXD decomposes the matrix via dynamically updating connected components. 

\looseness=-1
\subsubsection{Scalability under Multi-Threading}
To evaluate scalability, we vary the number of threads over 1, 2, 4 and 8.
Table \ref{tab:dyndxd_multithread} reports the total execution time of DynDXD across all instances.
Specifically, the total execution time decreases from 979.76 seconds with 1 thread to 502.12 seconds with 2 threads, achieving a $1.95\times$ speedup.
When increasing the number of threads to 4, the runtime is further reduced to 466.52 seconds, corresponding to a $2.10\times$ speedup over the 1-thread execution.
With 8 threads, the runtime reaches 401.24 seconds, yielding an overall $2.44\times$ speedup compared to 1 thread.
Fig. \ref{fig:dyndxd_multithread} compares the multi-threaded runtime of DynDXD across 10 instances during the search.
In these 10 instances, the 8-thread DynDXD attains a speedup close to $8\times$ compared to the 1-thread execution.

\begin{table}[!t]
\setlength\tabcolsep{1.5pt}
\centering
\small
\resizebox{\columnwidth}{!}{
\begin{tabular}{|l|c|c|c|c|c|c|c|c|c|c|c|c|c|}
\hline
\multirow{2}{*}{\textbf{Instance}} 
& \multirow{2}{*}{\textbf{\#Col}} 
& \multirow{2}{*}{\textbf{\#Row}} 
& \multicolumn{6}{c|}{\textbf{Time (s)}} 
& \multirow{2}{*}{\textbf{\#Sol}} 
& \multirow{2}{*}{\textbf{\#Sub}} 
& \multicolumn{3}{c|}{\textbf{Size}} \\
\cline{4-9} \cline{12-14}
\rule{0pt}{2.6ex}
& & 
& \textbf{DLX} 
& \textbf{DXZ} 
& \textbf{D3X} 
& \textbf{DynDXD} 
& \textbf{sharpSAT-TD} 
& \textbf{ExactMC} 
& &
& \textbf{$|$ZBDD$|$} 
& \textbf{$|$DNNF$|$} 
& \textbf{Ratio} \\
\hline
\rule{0pt}{2.6ex}

AsnetAm        & 531 & 2658 & TO &  2.66 & TO &  \textbf{1.35} & 309.81 &  66.43 & \num{1.9e+111} & 20 & 4848 &  \textbf{2443} & 0.504 \\ 
BtNorthAmerica & 528 & 2675 & TO &  2.48 & TO &  \textbf{0.57} & 259.23 &  30.51 & \num{8.5e+107} & 20 & 4804 &  \textbf{2421} & 0.504 \\ 
Dfn            & 536 & 2790 & TO &  2.65 & TO &  \textbf{0.57} & 300.18 &  33.26 & \num{7.5e+107} & 20 & 5114 &  \textbf{2576} & 0.504 \\ 
Garr200902     & 508 & 3038 & TO &  2.58 & TO &  \textbf{0.94} & 325.36 &  49.89 & \num{2.4e+125} & 20 & 5856 &  \textbf{2947} & 0.503 \\ 
Garr201001     & 491 & 2765 & TO & 12.42 & TO &  \textbf{1.28} & 229.98 &  98.23 & \num{8.3e+123} & 20 & 5304 &  \textbf{2671} & 0.504 \\ 
Garr201007     & 520 & 3012 & TO &  3.01 & TO &  \textbf{1.19} & 321.65 & 196.72 & \num{1.5e+122} & 20 & 5588 &  \textbf{2813} & 0.503 \\ 
Garr201008     & 518 & 2832 & TO &  3.37 & TO &  \textbf{0.85} & 304.04 &  63.83 & \num{1.6e+115} & 20 & 5374 &  \textbf{2706} & 0.504 \\ 
Garr201010     & 522 & 2785 & TO &  2.14 & TO &  \textbf{0.66} & 241.60 &  38.86 & \num{3.5e+115} & 20 & 5312 &  \textbf{2675} & 0.504 \\ 
Garr201104     & 532 & 2671 & TO &  3.43 & TO &  \textbf{0.60} & 289.24 &  31.23 & \num{4.3e+109} & 20 & 4958 &  \textbf{2498} & 0.504 \\ 
Garr201105     & 532 & 2787 & TO &  1.87 & TO &  \textbf{0.83} & 317.03 &  74.84 & \num{2.2e+112} & 20 & 5166 &  \textbf{2602} & 0.504 \\ 
Garr201108     & 533 & 2725 & TO &  3.40 & TO &  \textbf{1.04} & 293.23 &  58.01 & \num{9.8e+112} & 20 & 5060 &  \textbf{2549} & 0.504 \\ 
Garr201109     & 535 & 2986 & TO &  2.94 & TO &  \textbf{1.20} & 248.55 & 180.37 & \num{8.8e+117} & 20 & 5608 &  \textbf{2823} & 0.503 \\ 
Garr201110     & 537 & 2650 & TO &  1.64 & TO &  \textbf{0.59} & 268.77 &  79.35 & \num{3.6e+106} & 20 & 4662 &  \textbf{2350} & 0.504 \\ 
Garr201201     & 543 & 2814 & TO &  2.31 & TO &  \textbf{0.88} & 296.75 &  76.49 & \num{2.1e+106} & 20 & 5070 &  \textbf{2554} & 0.504 \\ 
Iij            & 458 & 2635 & TO &  1.66 & TO &  \textbf{0.61} & 214.25 &  34.11 & \num{1.1e+116} & 20 & 5046 &  \textbf{2542} & 0.504 \\ 
Latnet         & 521 & 2986 & TO &  2.88 & TO &  \textbf{0.83} & 299.92 & 172.03 & \num{3.0e+120} & 20 & 5596 &  \textbf{2817} & 0.503 \\ 
Missouri       & 532 & 2703 & TO &  1.52 & TO &  \textbf{0.68} & 199.55 &  33.12 & \num{3.5e+110} & 20 & 4864 &  \textbf{2451} & 0.504 \\ 
Palmetto       & 496 & 2716 & TO & 62.73 & TO &  \textbf{0.94} & 257.57 &  46.60 & \num{3.7e+113} & 20 & 5030 &  \textbf{2534} & 0.504 \\ 
Sinet          & 528 & 2773 & TO &  1.87 & TO &  \textbf{0.69} & 303.29 &  19.27 & \num{1.5e+108} & 20 & 5096 &  \textbf{2567} & 0.504 \\ 
Tinet          & 540 & 2806 & TO &  3.43 & TO &  \textbf{0.86} & 256.57 &  89.78 & \num{7.1e+110} & 20 & 5130 &  \textbf{2584} & 0.504 \\ 
Tw             & 550 & 3132 & TO &  8.33 & TO &  \textbf{2.29} & 288.17 & 542.90 & \num{6.2e+128} & 20 & 5838 &  \textbf{2938} & 0.503 \\ 
grafo10097.95  & 552 & 2828 & TO &  3.10 & TO &  \textbf{0.99} & 255.14 &  23.80 & \num{6.8e+109} & 20 & 5228 &  \textbf{2633} & 0.504 \\ 
grafo10106.100 & 557 & 2721 & TO &  2.32 & TO &  \textbf{0.89} & 217.16 &  38.54 & \num{7.2e+104} & 20 & 4962 &  \textbf{2500} & 0.504 \\ 
grafo10110.97  & 550 & 2919 & TO &  2.31 & TO &  \textbf{0.96} & 278.81 & 126.61 & \num{1.5e+113} & 20 & 5344 &  \textbf{2691} & 0.504 \\ 
grafo10115.92  & 550 & 2738 & TO &  3.31 & TO &  \textbf{1.03} & 273.21 &  28.53 & \num{4.0e+106} & 20 & 4902 &  \textbf{2470} & 0.504 \\ 
grafo10145.94  & 588 & 2816 & TO &  3.24 & TO &  \textbf{0.65} & 298.99 &  34.75 & \num{4.1e+105} & 20 & 5022 &  \textbf{2530} & 0.504 \\ 
grafo10158.91  & 562 & 2793 & TO &  3.86 & TO &  \textbf{2.12} & 324.65 &  41.88 & \num{8.3e+107} & 20 & 5196 &  \textbf{2617} & 0.504 \\ 
grafo10196.94  & 564 & 2843 & TO &  4.33 & TO &  \textbf{0.92} & 271.19 &  21.83 & \num{7.9e+108} & 20 & 5172 &  \textbf{2605} & 0.504 \\ 
grafo10198.94  & 564 & 2679 & TO &  2.03 & TO &  \textbf{0.92} & 265.39 &  15.99 & \num{6.5e+104} & 20 & 5004 &  \textbf{2521} & 0.504 \\ 
grafo10211.92  & 555 & 2896 & TO &  1.95 & TO &  \textbf{0.78} & 301.90 &  95.41 & \num{2.7e+112} & 20 & 5304 &  \textbf{2671} & 0.504 \\ 

\hline

\end{tabular}
}
\vspace*{1mm}
\caption{\small Experimental results on benchmark graphs.}
\label{tab:exact_cover_all}
\end{table}

\subsubsection{Overall Comparison}

We compare DynDXD using 8 threads with other excat cover solvers.
An overview of the experimental results is shown in Table \ref{tab:overall_instances}.
DynDXD solves the largest number of instances (462) within the time limit (1,500 seconds), followed by DXZ with 460 instances.
sharpSAT-TD and ExactMC solve 425 and 438 instances, respectively.
In comparison, DLX and D3X solve significantly fewer instances, with 182 and 183 instances solved.

\looseness=-1 
The experimental results are further analyzed in Table \ref{tab:exact_cover_all}, which reports the top 30 instances ranked by the number of solutions, excluding those instances where all solvers exceed the time limit.
Columns ``\#Col'' and ``\#Row'' indicate the number of columns and rows of the matrix, respectively.
Column ``Time'' reports runtimes of solvers in seconds, with `TO' denoting a timeout.
Column ``\#Sol'' records the number of exact covers, while ``\#Sub'' denotes the number of independent submatrices generated during the solving process.
Columns $|$ZBDD$|$ and $|$DNNF$|$ denote the sizes of the compilation representations by DXZ and DynDXD, respectively, and the ratio is defined as $|$DNNF$|/|$ZBDD$|$.
DynDXD is the fastest method on every instance reported in Table \ref{tab:exact_cover_all}.
For most instances with a large number of solutions, DLX and D3X result in timeouts because they enumerate all exact covers. 
DynDXD achieves an average speedup of 3.50$\times$ over DXZ, and produces smaller representations on all instances.
The compiled Decision-ZDNNFs have an average size of $50\%$ relative to the corresponding ZBDDs.
Moreover, DynDXD is up to nearly two orders of magnitude faster than SharpSAT-TD and ExactMC.

\section{Conclusion}
We have proposed a new parallel algorithm DXD which compiles the set of all exact covers into a decision-ZDNNF.
To further improve the efficiency of DXD, we have incorporated a dynamic algorithm for updating connected components under vertices deletion and insertion.
Experimental results show that DynDXD outperforms existing exact cover solvers as well as state-of-the-art model counters in terms of runtime performance.
Moreover, the Decision-ZDNNFs generated by DynDXD are consistently more compact than the corresponding ZBDDs constructed by DXZ.



\bibliography{sat2026}

@article{Akers1978,
  title={{Binary Decision Diagrams}},
  author={Sheldon B. Akers},
  journal={IEEE Transactions on Computers},
  volume={27},
  number={6},
  pages={509-516},
  year={1978}
}

@article{BeaLRS2017,
	title   = {{Exact Model Counting of Query Expressions: Limitations of Propositional Methods}},
	author  = {Paul Beame and Jerry Li and Sudeepa Roy and Dan Suciu},
	journal = {ACM Transactions on Database Systems},
	volume  = {42},
	number  = {1},
	pages   = {1--46},
	year    = {2017}
}

@article{BirL1999,
	author       = {Elazar Birnbaum and Eliezer L. Lozinskii},
	title        = {{The Good Old Davis-Putnam Procedure Helps Counting Models}},
	journal      = {Journal of Artificial Intelligence Research},
	volume       = {10},
	pages        = {457--477},
	year         = {1999}
}

@article{Bry1986,
  title={{Graph-based Algorithms for Boolean Function Manipulation}},
  author={Randal E. Bryant},
  journal={IEEE Transactions on Computers},
  volume={35},
  number={8},
  pages={677--691},
  year={1986}
}

@InCollection{CouMN2014,
  title={{Experimental evaluation of a branch and bound algorithm for computing pathwidth}},
  author={David Coudert and Dorian Mazauric and Nicolas Nisse},
  booktitle={Proceedings of the Thirtieth International Symposium on Experimental Algorithms (SEA-2014)},
  series={Lecture Notes in Computer Science},
  volume={8504},
  pages={46--58},
  year={2014},
  publisher={Springer}
}

@article{Dar2001,
	title   = {{Decomposable Negation Normal Form}},
	author  = {Adnan Darwiche},
	journal = {Journal of the ACM},
	volume  = {48},
	number  = {4},
	pages   = {608-647},
	year    = {2001}
}

@inproceedings{Dar2002,
  author    = {Adnan Darwiche},
  title     = {{A Compiler for Deterministic, Decomposable Negation Normal Form}},
  booktitle = {{Proceedings of the Eighteenth National Conference on Artificial Intelligence (AAAI-2002)}},
  pages     = {627-634},
  year      = {2002}
}

@article{GunM2011,
  title   ={{Entropy Minimization for Solving Sudoku}},
  author  ={Jake Gunther and Todd Moon},
  journal ={IEEE Transactions on Signal Processing},
  year    ={2011},
  volume  ={60},
  number  ={1},
  pages   ={508-513}
}

@article{HenK1999,
  title   = {{Randomized Fully Dynamic Graph Algorithms with Polylogarithmic Time per Operation}},
  author  = {Monika R. Henzinger and Valerie King},
  journal = {Journal of the ACM},
  volume  = {46},
  number  = {4},
  pages   = {502-516},
  year    = {1999}
}

@article{HitN1979,
	title={{A Technique for Implementing Backtrack Algorithms and Its Application}},
	author={Hirosi Hitotumatu and Kohei Noshita},
	journal={Information Processing Letters},
	volume={8},
	number={4},
	pages={174-175},
	year={1979}
}

@article{HopT1973,
	author = {John Hopcroft and Robert E. Tarjan},
	title = {{Algorithm 447: Efficient Algorithms for Graph Manipulation}},
	year = {1973},
	publisher = {Association for Computing Machinery},
	volume = {16},
	number = {6},
	journal = {Communications of the ACM},
	pages = {372-378}
}

@inproceedings{JrP2000,
	author       = {Roberto J. Bayardo Jr. and Joseph Daniel Pehoushek},
	title        = {{Counting Models Using Connected Components}},
	booktitle    = {{Proceedings of the Seventeenth National Conference on Artificial Intelligence (AAAI-2000)}},
	pages        = {157-162},
	year         = {2000}
}

@inproceedings{JunK2010,
  author={Tommi Junttila and Petteri Kaski},
  title={{Exact Cover via Satisfiability: An Empirical Study}},
  booktitle={{Proceedings of the Sixteenth International Conference on Principles and Practice of Constraint Programming (CP-2010)}},
  pages={297-304},
  year={2010}
}

@inproceedings{Karp1972,
  title     ={{Reducibility among Combinatorial Problems}},
  author    ={Richard M. Karp},
  booktitle ={Proceedings of a symposium on the Complexity of Computer Computations},
  pages     ={85-103},
  year      ={1972}
}

@article{KniNFBR2011,
  author  = {Simon Knight and Hung X. Nguyen and Nicklas Falkner and Rhys Bowden and Matthew Roughan},
  title   = {{The Internet Topology Zoo}},
  journal = {IEEE Journal on Selected Areas in Communications},
  volume  = {29},
  number  = {9},
  pages   = {1765-1775},
  year    = {2011}
}

@article{Knu2000,
  title   = {{Dancing Links}},
  author  = {Donald E. Knuth},
  journal = {Millennial Perspectives in Computer Science},
  year    = {2000},
  pages   = {187-214}
}

@inproceedings{Koi2006,
  author    ={Mikko Koivisto},
  title     ={{An $O^*(2^n)$ Algorithm for Graph Coloring and Other Partitioning Problems via Inclusion--Exclusion}},
  booktitle ={{Proceedings of the Forty-seventh Annual IEEE Symposium on Foundations of Computer Science (FOCS-2006)}},
  pages     ={583-590},
  year      ={2006}
}

@inproceedings{KorJ2021,
	author       = {Tuukka Korhonen and Matti J{\"{a}}rvisalo},
	title        = {{Integrating Tree Decompositions into Decision Heuristics of Propositional Model Counters}},
	booktitle    = {Proceedings of the Twenty-Seventh International Conference on Principles and Practice of Constraint	Programming (CP-2021)},
	pages        = {8:1-8:11},
	year         = {2021}
}

@inproceedings{LagM2017,
	title     = {{An Improved Decision-DNNF Compiler}},
	author    = {Jean-Marie Lagniez and Pierre Marquis},
	booktitle = {{Proceedings of the Twenty-Sixth International Joint Conference on Artificial Intelligence (IJCAI-2017)}},
	pages     = {667-673},
	year      = {2017}
}

@inproceedings{LaiMY2021,
  title={{The power of literal equivalence in model counting}},
  author={Yong Lai and Kuldeep S. Meel and Roland HC Yap},
  booktitle={Proceedings of the Thirty-Fifth AAAI Conference on Artificial Intelligence (AAAI-2021)},
  pages={3851--3859},
  year={2021}
}

@inproceedings{Min1993,
  author    ={Shin{-}ichi Minato},
  title     ={{Zero-suppressed BDDs for Set Manipulation in Combinatorial Problems}},
  booktitle ={{Proceedings of the Thirtieth Design Automation Conference (DAC-1993)}},
  pages     ={272-277},
  year      ={1993}
}

@inproceedings{Min1994,
	author       = {Shin{-}ichi Minato},
	title        = {{Calculation of Unate Cube Set Algebra Using Zero-Suppressed BDDs}},
	booktitle    = {{Proceedings of the Thirty-First Design Automation Conference (DAC-1994)}},
	pages        = {420-424},
	year         = {1994}
}

@inproceedings{NisYMN2017,
  author    ={Masaaki Nishino and Norihito Yasuda and Shin{-}ichi Minato and Masaaki Nagata},
  title     ={{Dancing with Decision Diagrams: a Combined Approach to Exact Cover}},
  booktitle ={{Proceedings of the Thirty-First AAAI Conference on Artificial Intelligence (AAAI-2017)}},
  pages     ={868-874},
  year      ={2017}
}

@inproceedings{NisYN2021,
  author={Masaaki Nishino and Norihito Yasuda and Kengo Nakamura},
  title={{Compressing Exact Cover Problems with Zero-suppressed Binary Decision Diagrams}},
  booktitle={{Proceedings of the Thirtieth International Joint Conference on Artificial Intelligence (IJCAI-2021)}},
  pages={1996-2004},
  year={2021}
}

@inproceedings{SangBB2004,
	title     = {{Combining Component Caching and Clause Learning for Effective Model Counting}},
	author    = {Tian Sang and Fahiem Bacchus and Paul Beame and Henry Kautz and Toniann Pitassi},
	booktitle = {{Proceedings of the Seventh International Conference on Theory and Applications of Satisfiability Testing (SAT-2004)}},
	year      = {2004}
}

@article{SleT1985,
  title     = {{Self-Adjusting Binary Search Trees}},
  author    = {Daniel D. Sleator and Robert E. Tarjan},
  journal   = {Journal of the ACM},
  volume    = {32},
  number    = {3},
  pages     = {652-686},
  year      = {1985}
}

@inproceedings{Thu2006,
	title     = {{sharpSAT - Counting Models with Advanced Component Caching and Implicit BCP}},
	author    = {Marc Thurley},
	booktitle = {{Proceedings of the Ninth International Conference on Theory and Applications of Satisfiability Testing (SAT-2006)}},
	pages     = {424-429},
	year      = {2006}
}

@article{HuaD2007,
	title   = {{The Language of Search}},
	author  = {Jinbo Huang and Adnan Darwiche},
	journal = {Journal of Artificial Intelligence Research},
	volume  = {29},
	pages   = {191-219},
	year    = {2007}
}

@book{Weg1987,
	title    = {{The Complexity of Boolean Functions}},
	author   = {Ingo Wegener},
	year     = {1987},
	publisher={John Wiley \& Sons, Inc.}
}

@book{Weg2000,
  title={{Branching Programs and Binary Decision Diagrams: Theory and Applications}},
  author={Ingo Wegener},
  year={2000},
  publisher={SIAM}
}

@book{Knu2016,
	title     = {{The Art of Computer Programming, Volume 4B: Combinatorial Algorithms}},
  author    = {Donald E. Knuth},
  year      = {2016},
  publisher = {Addison-Wesley Professional}
}

@article{DarM2002,
	title   = {{A Knowledge Compilation Map}},
	author  = {Adnan Darwiche and Pierre Marquis},
	journal = {Journal of Artificial Intelligence Research},
	volume  = {17},
	number  = {1},
	pages   = {229-264},
	year    = {2002}
}

@inproceedings{KimC1990,
  title={{A Parallel Algorithm for Constructing Binary Decision Diagrams}},
  author={Shinji Kimura and Edmund M. Clarke},
  booktitle={{Proceedings of the Fifth IEEE International Conference on Computer Design (ICCD-1990)}},
  pages={220-223},
  year={1990}
}

\appendix


\section{Splay Trees: Data Structure and Algorithms}

\looseness=-1
In this section, we introduce the data structure: splay tree and its algorithms that serve as the foundation on representing spanning trees.
A splay tree is a type of balanced search tree.
Each splay tree node contains the following five fields: $\tt value$, $\tt left$, $\tt right$, $\tt parent$ and $\tt size$.
The $\tt value$ field stores the value of this node.
The $\tt left$, $\tt right$ and $\tt parent$ fields, denote the left child, the right child and the parent of this node, used to maintain the structure of the balanced search tree.
The $\tt size$ field records the total number of nodes in the subtree rooted at this node.

\begin{algorithm}[t]\small
\caption{${\tt Rotate}(n)$}
\label{alg:rotate}
\KwIn{$n$: a node of a splay tree}

    $p \assign \parentfield[n]$

    $g \assign \parentfield[p]$

    \If{$n = \leftfield[p]$}{
        $\leftfield[p] \assign \rightfield[n]$

        \lIf{$\rightfield[n] \neq nil$}
        {
            $\parentfield[{\rightfield[n]}] \assign p$
        }
        $\rightfield[n] \assign p$
    }
    \Else
    {
        $\rightfield[p] \assign \leftfield[n]$

        \lIf{$\leftfield[n] \neq nil$}
        {
            $\parentfield[{\leftfield[n]}] \assign p$
        } 
        $\leftfield[n] \assign p$
    }

    $\parentfield[p] \assign n$

    $\parentfield[n] \assign g$

    \If{$g \neq nil$}{
        \lIf{$p = \leftfield[g]$}{
            $\leftfield[g] \assign n$
        }
        \lElse{
            $\rightfield[g] \assign n$
        }
    }

    $\sizefield[p] \assign \sizefield[{\leftfield[p]}] + \sizefield[{\rightfield[p]}] + 1$

    $\sizefield[n] \assign \sizefield[{\leftfield[n]}] + \sizefield[{\rightfield[n]}] + 1$
\end{algorithm}

\begin{algorithm}[t]\small
\caption{${\tt Splay}(n)$}
\label{alg:chroot}
\KwIn{$n$: a node of a splay tree}

    \While{$\parentfield[n] \neq nil$}{
        $p \assign \parentfield[n]$

        $g \assign \parentfield[p]$

        \lIf{$g = nil$}{
            ${\tt Rotate}(n)$
        }
        \ElseIf{($n = \leftfield[p]$ and $p = \leftfield[g]$) or
                ($n = \rightfield[p]$ and $p = \rightfield[g]$)}{
            ${\tt Rotate}(p)$

            ${\tt Rotate}(n)$
        }
        \Else{
            ${\tt Rotate}(n)$

            ${\tt Rotate}(n)$
        }
    }

\end{algorithm}

\looseness=-1
The two basic operations on splay trees are ${\tt Rotate}$ and ${\tt Splay}$, illustrated in Algorithm \ref{alg:rotate} and \ref{alg:chroot}, respectively.
The operation ${\tt Rotate}(n)$ moves a node $n$ to the position of its parent, while preserving the linear order.
It first locates the parent node $p$ and the grandparent node $g$ of $n$ (lines 1 \& 2).
If $n$ is the left child of $p$, it performs right rotation (lines 4 - 6).
Otherwise, it performs left rotation (lines 8 - 10).
Then, the parent node of $p$ is $n$, and the parent of $n$ is $g$ (lines 11 \& 12).
It adjust the child nodes of $g$.
If $p$ was the left child of $g$, the left child of $g$ is replaced by $n$ (line 14).
Otherwise, the right child of $g$ is replaced by $n$ (line 15).
Finally, the sizes of the splay trees rooted at $p$ and $n$ are recomputed, respectively (lines 16 \& 17).
The operation ${\tt Rotate}(n)$ requires constant time. 

\looseness=-1
The operation ${\tt Splay}(n)$ moves $n$ to the root of the splay tree.
Algorithm \ref{alg:chroot} repeatedly applies rotations to move node $n$ to the root of its splay tree.
The algorithm moves the node $n$ to the position of its parent until it becomes the root.
In each iteration, if the parent node $p$ has no parent (\ie, $g = \texttt{nil}$), a single rotation is performed to move $x$ directly above $p$ (line 4).
If both $n$ and its parent $p$ are on the same side of their respective parents (either both are left children or both are right children), the algorithm first
rotates $p$ above $g$, and then rotates $n$ above $p$ to preserve the relative structure while reducing the depth of $n$ (lines 6 \& 7).
Otherwise, $n$ and $p$ are located on opposite sides.
In this case, two consecutive rotations on $n$ are applied to restructure the tree and bring $n$ closer to the root (lines 9 \& 10).
The amortized complexity of operation ${\tt Splay}(n)$ is $O(\log |V|)$ where $|V|$ is the number of nodes of the splay tree.

\section{Algorithms for Euler Tour Representation of Spanning Trees}
\begin{algorithm}[t]\small
\caption{${\tt Concatenate}(n_u, n_v)$}
\label{alg:concat}
\KwIn{$n_u, n_v$: two root nodes of two different Euler tours $T_u$ and $T_v$}
\KwOut{$n$: the root node of the concatenation of $T_u$ and $T_v$}

    \lIf{$n_u = nil$}
    {
        \Return{$n_v$}
    }

    \lIf{$n_v = nil$}{
        \Return{$n_u$}
    }

    $n \assign$ the tail occurrence of $T_u$

    ${\tt Splay}(n)$

    $\rightfield[n] \assign n_v$

    $\parentfield[n_v] \assign n$

    $\sizefield[n] \assign \sizefield[n_u] + \sizefield[n_v]$

    \Return{$n$}

\end{algorithm}

\begin{algorithm}[t]\small
	\caption{${\tt AdjustHead}(u)$}
	\label{alg:AdjustHead}
	\KwIn{$u$: a vertex in an Euler tour $T_u$}
	\KwOut{$n$: the root node of an Euler tour with the head occurrence $u$}
	
	$u_1 \assign$ the first occurrence of $u$ in $T_u$
	
	${\tt Splay}(u_1)$
	
	$u_l \assign \leftfield[u_1]$
	
	$\leftfield[u_1] \assign nil$
	
	$\parentfield[u_l] \assign nil$
	
	$u_r \assign \rightfield[u_1]$
	
	$\rightfield[u_1] \assign nil$
	
	$\parentfield[u_r] \assign nil$
	
	Delete $u_1$ from $T_u$
	
	$u_2 \assign$ a new occurrence node of $u$
	
	$u_r \assign {\tt Concatenate}(u_r, u_2)$
	
	$n \assign {\tt Concatenate}(u_r, u_l)$
	
	\Return{$n$}
\end{algorithm}

\begin{algorithm}[t]\small
	\caption{${\tt Cut}(u, v)$}
	\label{alg:cut}
	\KwIn{$u, v$: two vertices in the same Euler tour $T$}
	\KwOut{$n_u, n_v$: two root nodes of two separated Euler tours with vertex $u$ and $v$}
	
	$u_e \assign$ the occurrence of $u$ s.t. the next occurrence is $v$
	
	${\tt Splay}(u_e)$
	
	$n_{ul} \assign \leftfield[u_e]$
	
	$n_{ur} \assign \rightfield[u_e]$
	
	\lIf{$n_{ul} \neq nil$}
	{
		$\parentfield[n_{ul}] \assign nil$
	}
	
	\lIf{$n_{ur} \neq nil$}
	{
		$\parentfield[n_{ur}] \assign nil$
	}
	
	Delete $u_e$ from $T$
	
	$v_e \assign$ the occurrence of $v$ s.t. the next occurrence is $u$
	
	${\tt Splay}(v_e)$
	
	$n_{vl} \assign \leftfield[v_e]$
	
	$n_{vr} \assign \rightfield[v_e]$
	
	\lIf{$n_{vl} \neq nil$}
	{
		$\parentfield[n_{vl}] \assign nil$
	}
	
	\lIf{$n_{vr} \neq nil$}
	{
		$\parentfield[n_{vr}] \assign nil$
	}
	
	Delete $v_e$ from $T$
	
	$n_u \assign {\tt Concatenate}(n_{ul}, n_{vr})$
	
	$n_v \assign n_{vl}$
	
	\Return{$n_u, n_v$}
	
\end{algorithm}

\begin{algorithm}[t]\small
	\caption{${\tt Link}(u, v)$}
	\label{alg:link}
	\KwIn{$u, v$: two vertices in two different Euler tours $T_u$ and $T_v$}
	\KwOut{$n$: the root node of a merged Euler tour}
	
	$n_u \assign {\tt AdjustHead}(u)$
	
	$n_v \assign {\tt AdjustHead}(v)$
	
	Create a node $n'_{u}$ representing the vertex $u$ of edge $(u, v)$
	
	Create a node $n'_v$ representing the vertex $v$ of edge $(v, u)$
	
	$n \assign {\tt Concatenate}(n_u, n'_v)$
	
	$n \assign {\tt Concatenate}(n, n'_u)$
	
	$n \assign {\tt Concatenate}(n, n_v)$
	
	\Return{$n$}
\end{algorithm}

\looseness=-1
We use Euler tours to represent spanning trees.
An Euler tour is a sequence $v_1 v_2 v_3 \cdots v_n$ that represents a traversal of the tree $T$ in which each edge is visited exactly twice.
Edge $(v_i, v_{i + 1})$ is on the tree $T$ for $1 \leq i \leq n - 1$.
Each appearance of a vertex $v$ in an Euler tour is called an \textit{occurrence} of $v$.
The first element of an Euler tour is called the \textit{head occurrence} while the last element is called the \textit{tail occurrence}.
A segment $S$ of the Euler tour for vertex $v$ is defined as the contiguous subsequence between the first and the last occurrence of $v$.
For example, consider the Euler tour $T: u_1 v_1  u_2 x_1 y_1 x_2 u_3$.
The head occurrence of is $u_1$ and the tail occurrence is $u_3$.
The segment for $x$ is $x_1 y_1 x_2$.
An Euler tour is represented by a splay tree.
For simplicity, we do not distinguish the two concepts of Euler tours and splay trees.

\looseness=-1
The four basic operations on Euler tours are ${\tt Concatenate}$, ${\tt AdjustHead}$, ${\tt Cut}$ and ${\tt Link}$, illustrated in Algorithm \ref{alg:concat}, \ref{alg:AdjustHead}, \ref{alg:cut} and \ref{alg:link}, respectively.
The operation ${\tt Concatenate}(n_u, n_v)$, illustrated in Algorithm \ref{alg:concat}, merges two Euler tours $T_u$ and $T_v$ in which two vertices $u$ and $v$ are, respectively.
If $n_u$ (resp. $n_v$) is an empty tree, the algorithm directly returns $n_v$ (resp. $n_u$) (lines 1 \& 2).
The algorithm then locates the tail occurrence $n$ of $T_u$ (line 3).
Then, it makes $n$ to the root of $T_u$ (line 4).
The Euler tour $T_v$ is appended to $T_u$ by attaching $n_v$ as the right child of $n_u$ and assign the parent of $n_v$ to be $n_u$ (lines 5 \& 6).
The size of the splay tree rooted at $n$ is the sum of the length of two Euler tours (line 7).
Finally, the root $n$ is returned (line 8).

\looseness=-1
The operation ${\tt AdjustHead}(u)$, illustrated in Algorithm \ref{alg:AdjustHead}, adjusts the Euler tour $T_u$ so that the first occurrence of vertex $u$ becomes the head occurrence.
The algorithm first locates the first occurrence $u_1$ of $u$ in $T_u$ (line 1).
Then, it make $u_1$ to be the root node of $T_u$ (line 2).
Next, the left child $u_l$ and the right child $u_r$ of $u_1$ are detached from $u_1$ (lines 3 - 8).
The node $u_1$ is deleted from $T_u$ (line 9).
A new occurrence node $u_2$ of vertex $u$ is attached to $u_r$ (line 10 \& 11).
Finally, $u_r$ and $u_l$ are concatenated, producing an Euler tour rooted at $n$ (line 12).
The algorithm returns $n$ (line 13).

\looseness=-1
The operation ${\tt Cut}(u, v)$, Algorithm \ref{alg:cut} splits an Euler tour $T$ into two separate Euler tours $T_u$ and $T_v$.
The algorithm first locates the occurrence $u_e$ of $u$ corresponding to edge $(u, v)$ (line 1), 
It then makes $_e$ to the root node of $T$.
The left subtree $n_{ul}$ and the right subtree $n_{ur}$ of $u_e$ are then detached from $T$ (lines 3 - 6).
The occurrence $u_e$ is removed from $T$ (line 7).
Next, the algorithm applies the same procedure to the removed tree $T$ for edge $(v, u)$ and obtain two subtrees $n_{vl}$ and $n_{vr}$ (lines 8 - 14).
The Euler tour rooted at $n_u$ is obtained by concatenating $n_{ul}$ and $n_{vr}$ (line 15).
The other rooted at $n_v$ is $n_{vl}$ (line 16).

\looseness=-1
The operation ${\tt Link}(u, v)$, illustrated in Algorithm \ref{alg:link}, merges two Euler tours corresponding to vertices $u$ and $v$ into a single Euler tour.
The algorithm first adjusts vertex $u$ and $v$ to be the root of the Euler tour $T_u$ and $T_v$, respectively (lines 1 \& 2). 
It then creates two occurrence nodes $n'_{u}$ and $n'_{v}$, representing vertex $u$ of edge $(u, v)$ and $v$ of edge $(v, u)$, respectively (lines 3 \& 4).
Next, the Euler tours $n_u$, $n'_v$, $n'_u$ and $n_v$ are concatenated as an Euler tour rooted at $n$ (lines 5 - 7). 

\end{document}